\documentclass[12pt]{article}
\usepackage{arxiv}
\usepackage{algorithm,algorithmic}
\usepackage{amsmath,amssymb,amsfonts}
\usepackage{todonotes}
\usepackage{float}
\usepackage{multicol,multirow}
\usepackage{secdot}
\usepackage{amsmath,array,ragged2e}
\newcolumntype{C}{>{\Centering\arraybackslash}m{0.14\linewidth}}
\usepackage{booktabs}
\usepackage{adjustbox}
\usepackage{subcaption}
\usepackage{graphicx}
\usepackage{hyperref}

\usepackage{balance} 

\newcommand{\algname}[1]{{\sf  #1}}
\DeclareMathOperator*{\argmin}{argmin} 
\newcommand{\bbR}{\mathbb{R}}

\usepackage{amsthm}

\newtheorem{definition}{Definition}
\newtheorem{remark}{Remark}
\newtheorem{theorem}{Theorem}
\newtheorem{conjecture}{Assumption}
\newtheorem{lemma}{Lemma}
\newtheorem{corollary}{Corollary}

\newtheorem*{proofidea}{Proof idea}

\begin{document}
\title{Functional multi-armed bandit and the best function identification problems}

\author{Yuriy Dorn\\
Lomonosov Moscow State University\\
Moscow Institute of Physics and Technology\\
Moscow, Russia\\
\texttt{dornyv@my.msu.ru}
\And 
Aleksandr Katrutsa\\
Skoltech, AIRI\\
Moscow, Russia\\
\texttt{amkatrutsa@gmail.com}
\AND
Ilgam Latypov\\
Lomonosov Moscow State University\\
Moscow Institute of Physics and Technology\\
Moscow, Russia\\
\texttt{i.latypov@iai.msu.ru}
\And 
Anastasiia Soboleva\\
Avito\\
Moscow, Russia\\
\texttt{soboleva.an@phystech.edu}
}

\maketitle

\begin{abstract}
We consider the model selection problem, where we have a set of candidate parametric functions and need to identify the function with the smallest minimum and corresponding minimizer. 
This problem arises in the competitive training of neural networks, where a set of candidates is given, and the limited computational budget prevents the use of a brute-force search. 
To address this problem, we propose generalizations of the classical multi-armed bandit (MAB) and best arm identification (BAI) setups, since using classical MAB and BAI setups leads to infeasible computational costs. 
We refer to the proposed setups as the \emph{functional} multi-armed bandit problem (FMAB) and the best \emph{function} identification (BFI) problems, respectively. 
For these problems, we establish lower regret bounds for different classes of candidate functions. 
To solve FMAB and BFI problems, we propose a novel reduction scheme to construct the \algname{F-LCB} algorithm, which is a UCB-type algorithm based on basic algorithms for nonlinear optimization with known convergence rates. 
The \algname{F-LCB} algorithm combines the arm selection step and the update of the current optimum approximation.
We provide regret upper bounds for \algname{F-LCB} based on the known convergence rates of the underlying base algorithms.
The regret upper bounds match with the derived lower bounds up to the logarithmic factor. 
Numerical experiments confirm that the proposed approach correctly identifies the optimal function and provides the minimizer for it in both smooth and non-smooth convex cases. 
Similarly, \algname{F-LCB} converges faster than SuccessiveHalving and Hyperband algorithms for the model selection problem, where the candidate functions are neural networks and only a stochastic gradient estimate is available.
\end{abstract}


\section{Introduction}
The stochastic MAB problem could be defined as follows: an agent chooses arm $A_{t}$ at each time step $t=1, \dots, T$ from the given set of arms $S = \{a_1, \dots, a_k\}$ and observes loss $l_t(A_{t})$. 
The agent can observe losses only for the chosen action at each step. 
This is referred to as \textit{bandit feedback}. 
For each arm $a$, the loss distribution $\mathcal{D}_{a}$ with expectation $\mathbb{E}_{x \sim \mathcal{D}_{a}}[x] = \mu(a)$ is fixed but unknown to the agent. 
At each round $t$, loss $l_t(A_t)$ is sampled from distribution~$\mathcal{D}_{A_t}$ independently after arm $A_t$ is chosen.
The agent's goal is to construct a learning algorithm that minimizes expected regret
\begin{equation}
\label{ex:reg_MAB}
   \mathbb{E}[R(T)] = \sum_{t=1}^T \left [ \mu(A_t) - \mu^*\right],
\end{equation}
where $\mu^* = \arg \min_{a \in S} \mu(a)$.
Surprisingly, there seem to be no works that properly generalize the multi-armed bandit setup to functions, where one models an unknown function as an arm rather than a random variable. 
This setup is appropriate for black-box optimization with multiple objectives involved.
For example, when developing an AI service, the appropriate model architecture and hyperparameter settings must be selected and optimized. 
The main challenge is that the optimal architecture or hyperparameter setting is unknown in advance. 
Thus, as in the MAB problem, one must explore different models and setups. 
However, exploration comes with costs that are negligible for small models but can be astronomically high for large-scale models like modern LLMs.
This challenge motivates modifications to the standard MAB setup, presented in Section~\ref{sec: FMAB}. 
We propose the functional multi-armed bandit and best function identification problems, where each arm corresponds to an unknown function and is equipped with a black-box oracle.

The main contributions of our study are the following:
\begin{itemize}
\item We propose the Functional Multi-armed Bandit (FMAB) and Best Function Identification (BFI) setups, which are appropriate for the model selection problem. 

\item We establish regret lower bounds for the introduced setups and provide particular forms for the different classes of candidate functions.

\item We develop the novel \algname{F-LCB} algorithm for FMAB and BFI problems and prove the regret upper bounds for general FMAB and deterministic BFI setups.

\item \algname{F-LCB} outperforms the SuccessiveHalving and Hyperband baselines in competitive neural networks training.  
\end{itemize}

\paragraph{Related works} The multi-armed bandit (MAB) problem has a rich history dating back to the seminal works of Thompson~\cite{thompson1933likelihood} and Robbins~\cite{robbins1952some}. Over decades, an enormous body of literature has accumulated, with various aspects of the problem covered in several foundational textbooks~\cite{cesa2006prediction, bubeck2012regret, slivkins2019introduction, lattimore2020bandit}. While there have been numerous attempts to generalize the MAB framework—particularly for hyperparameter optimization and complex decision spaces—most existing models primarily modify the feedback structure while maintaining the core correspondence between arms and random variables. For example, $\mathcal{X}$-bandits~\cite{bubeck2008online,bubeck2011x} generalize the MAB problem on an arbitrary measurable space of arms. In functional bandits~\cite{tran2014functional}, the agent plays arm $i$, the random variable $X_i$ is sampled, say $x_i^t$, and the value $f(x_i^t)$ is observed. 
Similarly, contextual bandits~\cite{woodroofe1979one, slivkins2011contextual, chu2011contextual} and Lipschitz bandits~\cite{agrawal1995continuum, kleinberg2008multi} assume that arms represent unknown distributions or functions of observable contexts.

In the domain of AutoML, several frameworks have gained prominence. Bayesian Optimization (BO)~\cite{snoek2012practical, shahriari2015taking} models the objective as a single black-box function, typically using Gaussian Processes. However, BO relies on shared structure across the parameter space and notoriously struggles in high-dimensional settings. In contrast, our Functional MAB (FMAB) framework assumes that arms represent separate functions, potentially from high-dimensional spaces, without requiring a shared global surrogate. While hybrid methods like BOHB~\cite{falkner2018bohb} combine BO with Successive Halving, they inherit these limitations when the underlying function landscapes are unrelated. Another influential approach is ASHA~\cite{li2020system}, an asynchronous variant of Hyperband that utilizes relative rankings for early stopping. While computationally efficient, ASHA operates as a heuristic ranking mechanism and does not exploit the specific mathematical properties of the convergence trajectories.

A distinct direction is represented by Population-Based Training (PBT)~\cite{jaderberg2017population} and its derivatives like PB2~\cite{parker2020provably}. PBT evolves a population of models by directly manipulating their parameters through explore/exploit operations and weight transfers (warm-starting). Our FMAB framework differs fundamentally as it allocates computation across independent optimization tasks without relying on weight sharing. Crucially, BO and PBT treat model selection as a purely black-box problem, ignoring the convergence guarantees of the underlying optimization routines, while our approach assumes that each arm corresponds to an optimization task with known convergence properties. This enables us to derive formal regret guarantees and identification bounds that are typically unavailable for purely heuristic AutoML methods.
To the best of our knowledge, this is the first tractable and theoretically grounded setup that generalizes the concept of arms to functions by leveraging the internal dynamics of the optimization process.

\section{Problem statement}
\label{sec: FMAB}
This section presents the functional modifications to the MAB problem and to the best-arm identification problems, respectively.
We denote these modifications as FMAB and BFI, formally introduce them, derive the lower bounds for the deterministic setting, and discuss an application that best fits the FMAB problem statement.

\subsection{Functional multi-armed bandit problem (FMAB)}
Given convex objective functions $f_1: \bbR^{n_1} \to \bbR$, \ldots , $f_K: \bbR^{n_K} \to \bbR$  and convex decision sets $\mathcal{X}_1, \dots, \mathcal{X}_K$ at each $t \in [0, T]$ round, the agent chooses index $i_t \in \{1, \dots, K\}$ and the decision vector $x^{t, i_t} \in \mathcal{X}_{i_t} \subseteq \bbR^{n_{i_t}}$; the agent receives oracle feedback $\mathcal{O}_{i_t}(x^{t, i_t})$, e.g., gradient $f'_{i_t}(x^{t, i_t})$ in the case of the first-order oracle.
, more details about oracle concept is presented in Section~\ref{oracles} in Appendix. 
The regret is defined as:
\begin{equation}
\label{ex:reg_o}
    R_O(T) = \sum_{t=1}^T \left [f_{i_t}(x^{t, i_t}) - f^* \right ],
\end{equation}
where $f^* = \min\limits_{1\leq i\leq K} \min\limits_{x \in \mathcal{X}_i} f_i(x)$. 
The agent aims to minimize regret $R_O$ through specific rules for selection index $i_t$ and decision vector~$x^{t, i_t}$.

The interpretation of optimized functions as arms appears in application-related works, such as ~\cite{zhang2023unified, wu2024budget}. 
The goal is to optimize multiple functions simultaneously with a limited compute budget. 
Model selection problem ~\cite{li2018hyperband}  is a particular case of this setting.
The discussion of relations of the introduced functional version of MAB setup and the classic MAB is prtesented in Section~\ref{sec::mab_vs_fmab}.

\subsection{Best function identification problem (BFI)}
Given convex objective functions $f_1: \bbR^{n_1} \to \bbR$, \ldots, $f_K: \bbR^{n_K} \to \bbR$  and convex decision sets $\mathcal{X}_1, \dots, \mathcal{X}_K$ at each $t \in [0, T]$ round, the agent chooses index $i_t \in \{1, \dots, K\}$ and the decision vector $x^{t, i_t} \in \mathcal{X}_{i_t} \subseteq \bbR^{n_{i_t}}$; the agent observes the loss $f_{i_t}(x^{t, i_t})$. 
We assume that the agent has access to the oracles $\mathcal{O}_i(x)$ for each objective function~$f_i$ and the oracle is the only source of information provided for each subproblem $\mathcal{P}_i$ defined by $f_i$ and $\mathcal{X}_i$ (i.e., we use the black-box assumption). 
At the end of $T$ rounds, the agent selects an arm, denoted by $J_T$, and aims to minimize the regret $R_B$ defined as:
\begin{equation}
\label{ex:reg_b}
    R_B(T) = \min_{x \in \mathcal{X}_{J_T}} f_{J_T} (x) - f^*,
\end{equation}
where $f^* = \min\limits_{1\leq i\leq K} \min\limits_{x \in \mathcal{X}_i} f_i(x)$. 
We call this problem to as best function identification problem (BFI), which is 
is an analog for the well-known best arm identification problem~\cite{audibert2010best}.

\subsection{Relation between MAB and FMAB settings}
\label{sec::mab_vs_fmab}

Let us recap the stochastic MAB problem setup with time horizon~$T$ and $K$ arms, each arm $i$ is equipped with an unknown distribution~$\mathcal{D}_i$ with an expected value $\mu_i>0$. 
At each step~$t$ the agent chooses arm $i_t$ and observes reward $\mu_{i_t}+\xi_{t}$ sampled from~$\mathcal{D}_{i_t}$, where the noise $\xi_t$, is unbiased $\mathbb{E}[\xi_t] = 0$ by construction. 
The regret is defined by (\ref{ex:reg_MAB}).
Let us show how FMAB models this setup.

\paragraph{MAB as FMAB with zero-order oracle.}
Consider FMAB setting with $f_i(x) = -\mu_i$ equipped by the zero-order oracle $\mathcal{O}_i(x_t)$ that provides noised observations $ \mathcal{O}_i(x_t) = -\mu_i-\xi_t = f_i(x_t) -\xi_t$, sampled from $\mathcal{D}_i$ with "minus", with $\mathbb{E}[\mathcal{O}_i(x_t)] = f_i(x_t)$.  
Then regret (\ref{ex:reg_o}) is exactly the same as in (\ref{ex:reg_MAB}). This direct reduction is formally correct. 
Let us further consider a reduction scheme that utilizes first-order oracles in FMAB.

\paragraph{MAB as FMAB with first-order oracle}
Consider FMAB setting with $f_i(x) = \frac{1}{2}(x-\mu_i)^2-\frac{\mu_i^2}{2}$. 
Here $x$ estimates the expected value~$\mu_i$.
Note that this setting has a few nice properties:
\begin{itemize}
    \item Optimal objective values represent arms: $f_i^* = -\frac{\mu_i^2}{2}$,
    \item Objective value $f_i(x_t) = \frac{x^2_t}{2} - x_t \mu_i$ can be estimated via tractable $\hat{f}_i(x_t) = \frac{x_t^2}{2} - x_t(\mu_i+\xi_t)$, where $\mu_i+\xi_t$ is a sampled from reward distribution~$\mathcal{D}_i$. 
    Such estimation is unbiased $\mathbb{E}[\hat{f}_i(x) ] = f_i(x)$.
    \item We can observe $\mathcal{O}_i(x_t) = x_t - \mu_i - \xi_t$, which is  unbiased gradient estimation $\mathbb{E}[\mathcal{O}_i(x_t)] = x_t - \mu_i = f'_i(x)$.
\end{itemize}
In this approach, regret $R_O$~(\ref{ex:reg_o}) does not explicitly represent regret for MAB setup~(\ref{ex:reg_MAB}).

\paragraph{Reduction of FMAB to MAB} 
We have shown how MAB can be treated as particular cases of FMAB.
However, the natural question of whether FMAB can handle cases which are intractable by MAB is still actual.
By definition of FMAB, it could be represented as MAB with a continuous number of arms indexed by pair $(i, x)$, where $i\in \{1, \dots, K\}$ and $x\in \bbR^{n_i}$ or approximated by MAB with an infinite number of arms via discretisation for the continuous space of~$x$. 
Both cases require additional structural assumptions on arms to become tractable (see \cite{wang2008algorithms} as an example). If these assumptions are in place, then, just like in Bayesian optimization (BO) algorithms, sampled arms make decisions based on feedback. This is not avoidable in general, but in FMAB we assume that each arm $i$ is equipped with a learning algorithm $\mathcal{A}_i$, that could optimize the corresponding function $f_i(x)$. These assumptions are natural for many applications (see the next subsection as an example) and are new and not considered in the general MAB setting. It allows, in general, to learn much faster compared to classical MAB or BO algorithms. So  FMAB could be solved as MAB, but it is much better to solve it as FMAB to avoid sub-exponential costs.

\subsection{Applications}

\paragraph{Competitive neural network training}
Modern neural networks are very costly to train~\cite{gusak2022survey}. 
Therefore, the standard trial-and-run approach is very inefficient.  
Within our framework, sequential training of all candidate models could be avoided and the most accurate model is identified automatically.
Assume that there are $k$ candidate models. 
Each model $i \in \{1, \dots, k\}$ is denoted by the number of parameters $n_i$, feasible decision set $\mathcal{X}_i \subseteq \bbR^{n_i}$ and domain-specific quality metric w.r.t. training cost $f_i: \mathcal{X}_i \rightarrow \bbR$. 
Then, regret~$R_O$ represents the sum of training costs for the optimal model and costs for experiments for other models.




\paragraph{Optimization method selection}
Modern large-scale optimization tasks often present a dilemma: the most suitable algorithm for a given problem is unknown a priori. 
While theoretical convergence rates are well established, practical performance depends heavily on the structural properties of the data, such as sparsity, condition number, and feature space dimensionality. 
In this regime, researchers face trade-offs among model and solver classes. 
For instance, in high-dimensional settings, one must often choose between an $\ell_1$-regularized model solved via FISTA~\cite{beck2009fast} and a dense counterpart utilizing variance-reduction techniques such as SVRG~\cite{johnson2013accelerating}. 
The efficiency of such choices hinges on the underlying signal sparsity, which is typically not observable before training. 

A prominent example of such complexity is found in optimal transport, where large-scale problems admit multiple formulations and a wide array of solvers~\cite{tupitsa2022numerical}. 
The choice among primal-dual accelerated methods, Sinkhorn-based iterations, and stochastic approaches often depends on the desired precision and the strength of regularization.

\paragraph{Convex relaxations} Our framework is applicable if a complex original problem admits a convex relaxation. 
Such approximations are vital for a broad spectrum of tasks, including the weighted maxmin dispersion problem~\cite{haines2013convex}, optimal distributed control~\cite{fazelnia2016convex}, mixed-integer programs~\cite{saxena2010convex}, and polynomial optimization~\cite{keller2014convex}. 
Since these convex relaxations are solved using iterative optimization methods, their performance is governed by analytical convergence bounds, such as $O(1/t^2)$ for accelerated gradient schemes if objective function is $L$-smooth.
Our \algname{F-LCB} algorithm leverages these properties to evaluate the quality of a relaxation ``on the fly'', enabling the early identification of the most promising problem formulation and the optimal allocation of computational resources.



\subsection{Notation for function classes}
\label{sec:notation}

We summarize here the main notation and function classes used throughout the paper.  
For each $i \in \{1,\dots, K\}$, let $f_i$ be the objective function defined on a convex domain $\mathcal{X}_i \subset \mathbb{R}^{n_i}$ and optimized by a base algorithm $\mathcal{A}_i$.  
We denote $f^* \;=\; \min_{1 \leq i \leq K}\ \min_{x \in \mathcal{X}_i} f_i(x)$, and let $i^*$ be the corresponding optimal index.   The diameter of $\mathcal{X}_i$ is denoted by $R_i \;=\; \sup_{x,y \in \mathcal{X}_i} \|x - y\|_2$, and we also define the global diameter $R \;=\; \max_{1 \le i \le K} R_i$.

For each convex function $f_i$, we assume that
\begin{equation*}
    \frac{\mu_i}{2}\|x - y\|_2^2 
    \ \le\ f_i(y) - f_i(x) - \langle f'_i(x), y - x\rangle 
    \ \le\ \frac{L_i}{2}\|x - y\|_2^2 + M_i \|x - y\|_2,
\end{equation*}
where $\mu_i \ge 0$ is the strong convexity constant, $L_i \ge 0$ is the smoothness parameter, and $M_i \ge 0$ is the Lipschitz constant of $f_i$. 
Different combinations of $(\mu_i, L_i, M_i)$ correspond to different standard classes of convex optimization problems:
\begin{itemize}
    \item if $M_i > 0$, $\mu_i = 0$, $f_i$ is convex $M_i$-Lipschitz function
    \item if $L_i > 0$ and $\mu_i = 0$, $f_i$ is $L_i$-smooth convex function
    \item  If $\mu_i > 0$, $M_i > 0$, $f_i$ is $\mu_i$-strongly convex and $M_i$-Lipschitz
    \item If  $\mu_i > 0$, $L_i > 0$, $f_i$ is $\mu_i$-strongly convex and $L_i$-smooth.
\end{itemize}
For convenience, we introduce the following global parameters:
\begin{equation*}
    M = \max_i M_i, \qquad 
    L = \max_i L_i, \qquad 
    \mu = \min_i \mu_i, \qquad 
    \kappa = \max_i \frac{L_i}{\mu_i}.
\end{equation*}

\section{Lower bounds for deterministic setups}
\label{sec:lower-bounds}

We present minimax lower bounds for FMAB and BFI problems in deterministic settings.  
The idea is to reduce these problems to standard optimization tasks with known lower bounds~\cite{nemirovskij1983problem, nesterov1983method}.

We use $\mathcal{F}$ to denote a class of optimization problems defined by a family of objective functions $\{f_i\}_{i=1}^K$, their corresponding domains $\{\mathcal{X}_i\}_{i=1}^K$, and oracle types. 
We assume that all problems belong to the same class $\mathcal{F}$, i.e. they share the same structural properties (e.g., convexity, smoothness, strong convexity) and oracle complexity.

Consider a family of optimization problems $\mathcal{P}_i$:
\begin{equation}
    \min_{x \in \mathcal{X}_i} f_i(x), 
    \qquad i=1,\dots,K,
\end{equation}
where all $f_i$ belong to the same class $\mathcal{F}$ and have the oracles with the same complexity.  
For such problems, the classical minimax lower bound states that for any algorithm $\mathcal{A}$ and any $t\in\mathbb{N}$ there exists a problem instance such that
\begin{equation}\label{eq:lb-single}
    f(x_t) - \min_{x \in \mathcal{X}} f(x) \ \ge\ g(H,t),
\end{equation}
where $g(H,t)$ is the complexity function depending on $H=(f,\mathcal{D})$.  

For many standard classes, $g(H,t)$ factorizes as
\begin{equation}
    g(H,t) \;=\; \phi(H)\, t^{-\alpha},
\end{equation}
allowing us to define the class-wide hardness function
\begin{equation}
    \underline g(t)=\inf_H g(H,t),
\end{equation}
typically of order $\Omega(t^{-\alpha})$.

\begin{theorem}[Minimax lower bound for BFI]
\label{thm:lb-bfi}
For any BFI algorithm and any $T \in \mathbb{N}$, there exists a family 
$\{P_i\}_{i=1}^K \subset \mathcal{F}$ such that
\begin{equation}
    R_B(T) \ \ge\ \underline g^{-1}\!\left(\frac{T}{K}\right),
\end{equation}
where $\underline g^{-1}$ is the functional inverse of $\underline g$.
\end{theorem}

\begin{theorem}[Minimax lower bound for FMAB]
\label{thm:lb-fmab}
For any FMAB algorithm and any $T \in \mathbb{N}$, there exists a family 
$\{P_i\}_{i=1}^K \subset \mathcal{F}$ such that
\begin{equation}
  R_O(T)\ \ge\ \inf_{\{k_i \ge 0:\ \sum_{i=1}^K k_i = T\}}
  \ \sum_{i=1}^K G(k_i),
\end{equation}
where 
\begin{equation}
    G(m) = \sum_{s=1}^m \underline g(s).
\end{equation}
\end{theorem}

For homogeneous problems (all $P_i$ have the same hardness $\underline g$), this yields the following orders:

\begin{itemize}
    \item Convex $M$-Lipschitz: $\underline g(s)=\Omega(MR/\sqrt{s}) \ \Rightarrow\ R_O(T)=\Omega(MR\sqrt{T})$,
    \item $L$-smooth convex: $\underline g(s)=\Omega(LR^2/s^2) \ \Rightarrow\ R_O(T)=\Omega(LR^2)$,
    \item $\mu$-strongly convex, $M$-Lipschitz: $\underline g(s)=\Omega(M^2/(\mu s)) \ \Rightarrow\ R_O(T)=\Omega\bigl(\frac{M^2}{\mu}\log T\bigr)$,
    \item $\mu$-strongly convex, $L$-smooth: $\underline g(s)=\Omega(R^2 e^{-s/\sqrt{\kappa}}) \ \Rightarrow\ R_O(T)=\Omega(R^2)$.
\end{itemize}

These lower bounds match, up to logarithmic factors, the upper bounds we derive later for the proposed \algname{F-LCB} algorithm.

\begin{remark}
Unlike static "explore-first" schedules that are optimal only for indistinguishable worst-case scenarios, \algname{F-LCB} adaptively focuses computational effort on promising arms. It reacts immediately when an arm's LCB becomes suboptimal, avoiding the resource waste inherent in static allocations.
\end{remark}


\section{F-LCB algorithm}
\label{sec: F-LCB}

This section presents the novel \algname{F-LCB} algorithm to solve the stated FMAB~(\ref{ex:reg_o}) and BFI~(\ref{ex:reg_b}) problems and proof of the corresponding regret rates.
However, for the reader's convenience, we first introduce the necessary notations important for further presentation.

\begin{definition}
\label{def::gkdelta-alg}
An algorithm
\[
x_{k+1} = \mathcal{A}\left (x_0, \mathcal{O}(x_0), \dots, x_k, \mathcal{O}(x_k) \right )
\]
An algorithm is called a $g(k, \delta)$-bounded algorithm if, for any $k \in \mathbb{N}$ and $\delta>0$ inequality
\[
f(x_k) - f(x^*) \leq g(k, \delta)
\]
holds with a probability of at least $1-\delta$.
If there exists a function $g(k)$ such that 
$
f(x_k) - f(x^*) \leq g(k),
$
we say that the algorithm $\mathcal{A}$ is $g(k)$-bounded.
\end{definition}

Function $g(k, \delta)$ (or $g(k)$ in the deterministic case) represents the convergence rate for algorithm $\mathcal{A}$. 
The notation $g(k)$ is more convenient for deterministic algorithms with exact oracles, while $g(k, \delta)$ is more appropriate for stochastic methods or methods utilizing inexact oracles.
Now, we are ready to present our \algname{F-LCB} algorithm for both FMAB and BFI problems, taking $g(k, \delta)$- or $g(k)$-bounded algorithms as the main ingredient; see Algorithm~\ref{alg:f_ucb}.

\begin{algorithm}[!ht]
\caption{\algname{F-LCB} algorithm}
\label{alg:f_ucb}
\begin{algorithmic}[1]
\REQUIRE{number of functions $K$, $g_i(k, \delta)$-bounded optimization method $\mathcal{A}_i$ for $i=1,\ldots, K$,  period $T$, initial estimates $x_0^{\mathcal{P}_1}, \dots, x_0^{\mathcal{P}_K}$, parameter $\delta$ ($\delta=0$ for deterministic setup).}
    \STATE Run $\mathcal{A}_i$ for each function $i$ ($i=1, \dots, K$) to compute $x_1^{\mathcal{P}_i} = \mathcal{A}_i(x_0^{\mathcal{P}_i}, \mathcal{O}_{\mathcal{P}_i} (x_0^{\mathcal{P}_i}))$.
    \STATE For each function $i$ ($i=1, \dots, K$) set $k_i=1$ and initialize $LCB_i(k_i, \delta) = f_i(x_1^{\mathcal{P}_i}) - g_i(k_i  , \delta)$.
    \FOR{$t=1, \dots, T$}
    \STATE Choose function $i_t = \argmin\limits_{1 \leq i \leq K} LCB_i(k_i, \delta)$.
    \STATE Compute $$x_{k_{i_t}+1}^{\mathcal{P}_{i_t}} = \mathcal{A}_{i_t}(x_0^{\mathcal{P}_{i_t}}, \mathcal{O}_{\mathcal{P}_{i_t}}(x_0^{\mathcal{P}_{i_t}}), \ldots, x_{k_{i_t}}^{\mathcal{P}_{i_t}}, \mathcal{O}_{\mathcal{P}_{i_t}}(x_{k_{i_t}}^{\mathcal{P}_{i_t}})).$$
    \STATE Update LCB index of the played function and preserve others:  
    \[
    LCB_{i_t}(k_{i} + 1, \delta) = \begin{cases}
    LCB_i(k_i, \delta), & i \neq i_t,\\
    f_{i_t}(x_{k_{i_t}+1}^{\mathcal{P}_{i_t}}) - g_{i_t}(k_{i_t}+1, \delta), & i = i_t.
    \end{cases}
    \]
    \IF{$g_{i_t}(k_{i_t+1}, \delta) < \frac{\varepsilon}{2}$} 
    \STATE return $f_{i_t}$ 
    \ENDIF
    \STATE Increase iteration counter for the played arm: $k_{i_t} := k_{i_t}+1$.
    \ENDFOR
\end{algorithmic}
\end{algorithm}

The main idea of Algorithm~\ref{alg:f_ucb} is to treat base optimization algorithm's convergence rate as confidence intervals to construct the lower confidence bound on the objective value of the chosen arm. 
So, the overall scheme is as follows: each optimization problem~$\mathcal{P}_i$ defined by $(f_i, \mathcal{X}_i)$ equipped with $g_i(k, \delta)$-bounded algorithm~$\mathcal{A}_i$, suitable for~$\mathcal{P}_i$ problem class. 
Then, at each time step $t$, our algorithm chooses the $i_t$-th arm. 
Therefore, we run an iteration of~$\mathcal{A}_{i_t}$ based on the current optimistic estimation $LCB_i(x_t^{\mathcal{P}_i})$ of the corresponding objectives' optimal values $f_{i_t}^*$.

\begin{remark}
If $f_i(x)$, $1\leq i \leq K$, is accessed by an inexact oracle~$\mathcal{O}_i$, the oracle should be sampled multiple times at each point. 
The corresponding algorithm $\mathcal{A}_i$ usually controls the number of samples. 
\end{remark}

\begin{remark}
In \algname{F-LCB}, each iteration simultaneously reduces uncertainty and expected regret. Thus, the LCB index acts as a unified criterion that is both regret-optimal and information-optimal (up to logarithmic factors), allowing the BFI stopping criterion to integrate naturally without distinct exploration phases.
\end{remark}


This is a direct application of ideas introduced in the seminal paper~\cite{auer2002using} if one uses convergence rates for optimization algorithms instead of concentration rates of statistical estimators. 
This approach was proposed in~\cite{dorn2024fast} for MAB with heavy tails.
Note that one could use different base optimization algorithms $\mathcal{A}_i$ for different $i$.
Next, we present the regret rates for the \algname{F-LCB} algorithm in Table~\ref{tab:summary}, which are proved formally in Sections~\ref{sec::determinstic_theory}.

\begin{table}[!ht]
    \centering
    \caption{Summary of regrets $R_O$ and $R_B$ rates for BFI and FMAB problems in the deterministic setup. 
    The proofs for these rates are given in Appendix~\ref{sec::proof_det_table}. 
    Here, we assume that functions $f_i$ belong to the same class and base optimizers $\mathcal{A}_i$ are the same and equal to those reported in column 2. PGD denotes Projected Gradient Descent, and AGD denotes Accelerated Gradient Descent, and $\kappa = \frac{L}{\mu}.$}
    \resizebox{\linewidth}{!}{
    \begin{tabular}{cccccc}
    \toprule
       Function  & Base optimizer & $g(k)$ & \# iter for $R_B \leq \varepsilon$ & $R_O(T)$ \\
       \midrule
        Convex $M$-Lipschitz  & \algname{PGD} & $\frac{RM}{\sqrt{k}}$ & $\sum\limits_{i=1}^K\Bigl \lceil\frac{M_i^2R_i^2}{\max(f_i^*-f^*-\frac{\varepsilon}{2},\frac{\varepsilon}{2})^2}\Bigr \rceil$& $O\left(\sqrt{T \cdot \sum_{i=1}^K M_i^2 R_i^2} \right)$ \\
        Convex $L$-smooth  & \algname{AGD} & $\frac{LR^2}{k^2}$ & $\sum\limits_{i=1}^K\left\lceil \sqrt{\frac{L_iR_i^2} {\max(f_i^*-f^*-\frac{\varepsilon}{2}, \frac{\varepsilon}{2})}} \; \right\rceil$ &  $O\left(\sum\limits_{i=1}^KL_iR_i^2\right)$ \\
        $\mu$-strongly convex $M$-Lipschitz  & \algname{PGD} & $\frac{M^2}{\mu k}$ & $\sum\limits_{i=1}^K\Bigl \lceil\frac{M_i^2}{\mu_i \max(f_i^*-f^*-\frac{\varepsilon}{2}, \frac{\varepsilon}{2})}\Bigr \rceil$ & $O\left(\left(\sum\limits_{i=1}^K \frac{M_i^2}{\mu_i}\right) \log T\right)$\\
         $\mu$-strongly convex $L$-smooth  & \algname{AGD} & $R^2 \exp\{-\frac{k}{\sqrt{\kappa}}\}$ & $\sum\limits_{i=1}^K\Bigl \lceil\sqrt{\kappa}\log\left(\frac{R_i^2}{\max(f_i^*-f^*-\frac{\varepsilon}{2}, \frac{\varepsilon}{2})}\right)\Bigr \rceil$ & $O\left(\sum\limits_{i=1}^K \frac{R_i^2}{\exp\left\{ \frac{1}{\sqrt{\kappa_i}}\right\} - 1}\right)$\\
         \bottomrule
    \end{tabular}
    }
    \label{tab:summary}
\end{table}


\subsection{Determenistic case}
\label{sec::determinstic_theory}
This section presents the regret bounds for FMAB and BFI problems in terms of the convergence rates~$g_i(k)$ for the base optimizers $\mathcal{A}_i$ equipped with the deterministic oracles.
After that, the substitution of the particular forms of $g_i(k) = g(k)$ corresponding to the base optimizer (AGD or PGD) leads to the regret bounds from Table~\ref{tab:summary}. 

\paragraph{FMAB} 
This paragraph focuses on the deterministic FMAB problem, which aims to minimize the regret $R_O$~(\ref{ex:reg_o}).
Lemma~\ref{lemma:o_regret_bound} shows how to bound $R_O$ in general for $g_i(k)$-bounded base optimizers.
Theorem~\ref{thm:regret_alpha} specializes this result to the case where all functions $g_i(k)$ decrease at the same polynomial rate, i.e., scale with a common convergence \textbf{r}ate $r$.

\begin{lemma}
\label{lemma:o_regret_bound}
Assume $\mathcal{A}_i$ ($i=1,\ldots,K$) be  $g_i(k)$-bounded base algorithms. 
Then for Algorithm~\ref{alg:f_ucb} for all $\tau \in \overline{1, T}$ holds:
\begin{equation}
R_O(\tau) \leq \sum^{\tau}_{t = 1}g_{i_t}(k_{i_t, t}) =\sum_{i=1}^{K} \sum_{k = 1}^{k_{i, \tau}} g_i(k),
\label{eq::ro_fmab_det}
\end{equation}
where $k_{i, t}$ is a number of calls for the $i$-th function by time $t$.
\end{lemma}

\begin{proof}
Let $i_t$ be the arm selected at time $t$. 
Then, its LCB value is the smallest one among the arms. 
That is, for all $j$:
\begin{equation*}
    LCB_{i_t}(k_{i_t, t}) =  f_{i_t}(x^{i_t, k_{i_t, t}}) - g_{i_t}( k_{i_t, t})  \leq f_j(x^{j, k_{j, t}}) - g_j(k_{j, t}) \leq f_j^*.
\end{equation*}

In particular, this holds w.r.t. the best arm, yielding an estimate of the per-step regret:
\begin{equation}
    f_{i_t}(x^{i_t, k_{i_t, t}}) -g_{i_t}(k_{i_t, t}) \leq f^* \Rightarrow f_{i_t}(x^{i_t, k_{i_t, t}}) - f^* \overset{(\star)}{\leq} g_{i_t}(k_{i_t, t}).
\end{equation}

Summing up inequality ($\star$) for $t=1, \ldots, \tau$ we get regret rate:
\begin{equation}
    \sum^{\tau}_{t = 1}f_{i_t}(x^{i_t, k_{i_t}}) - f^* \overset{(1)}{\leq} \sum^{\tau}_{t = 1} g_{i_t}(k_{i_t}) \overset{(2)}{=} \sum_{i=1}^{K} \sum_{t= 1}^{k_{i,\tau}} g_i(t).
\end{equation}
Equality (2) is obtained by grouping terms over arms.
\end{proof}

\begin{theorem}
\label{thm:regret_alpha}
    Let $r > 0$. 
    Assume $\mathcal{A}_i$ ($i=1,\ldots,K$) are $g_i(k)$-bounded base algorithms with $g_i(k) = \frac{\beta_i}{k^r}$.
    Then for Algorithm~\ref{alg:f_ucb}, for all $\tau \in \overline{1, T}$ the following regret bounds hold:
    \begin{itemize}
        \item if $r \in (0, 1)$: $R_O(\tau) \leq  O\left(\left(\sum_{i=1}^K \beta_i^{\frac{1}{r}}\right)^{r} \tau^{1-r}\right)$;
        \item if $r = 1$: $R_O(\tau) \leq  O\left(\sum_{i=1}^K \beta_i \log \tau \right)$;
        \item if $r > 1$: $R_O(\tau) \leq  O\left(\sum_{i=1}^K \beta_i \right)$.
    \end{itemize}
\end{theorem}

\begin{proof}
For each base algorithm, we consider the inner sum over number of updates $\sum_{k=1}^{k_{i,\tau}} \frac{1}{k^r}$.

For $r \in (0,1)$, by the standard summation bound 
$ \sum_{k=1}^{k_{i,\tau}} \frac{1}{k^r} = O\bigl(k_{i,\tau}^{1-r}\bigr)$.
Then, using Hölder's inequality together with the budget constraint $\sum_{i=1}^K k_{i,\tau} \le \tau$, we obtain
\begin{equation*}
\sum_{i=1}^K \beta_i \sum_{k=1}^{k_{i,\tau}} \frac{1}{k^r} 
= O\left(\sum_{i=1}^K \beta_i k_{i,\tau}^{1-r}\right)
\le \left(\sum_{i=1}^K \beta_i^{\frac{1}{r}}\right)^{r} \tau^{1 -r}.
\end{equation*}
For $r = 1$, we have  $\sum_{k=1}^{k_{i,\tau}} \frac{1}{k} = O(\log k_{i,\tau}) \le O(\log \tau)$, which gives
$ R_O(\tau) \le O\left(\sum_{i=1}^K \beta_i \log \tau\right)$.
For $r > 1$, the sum is bounded by a constant, so $R_O(\tau) \le O\left(\sum_{i=1}^K \beta_i \right)$.
\end{proof}

\paragraph{BFI}
This paragraph focuses on the deterministic BFI problem, which aims to minimize the regret~$R_B$~(\ref{ex:reg_b}).
Theorem~\ref{th:bfi_det_bound} shows how many steps are required for Algorithm~\ref{alg:f_ucb} to get $R_B$ smaller than $\varepsilon$ from the selected $g_i(k)$-bounded base optimizers.

\begin{theorem}
\label{th:bfi_det_bound} 
Consider a deterministic BFI problem. 
We denote by $f^* = \min_{1\leq i\leq k} f_i^*$. 
To achieve regret $R_B(T) = \min\limits_{x \in \mathcal{D}_{J_T}} f_{J_T} (x) - f^*\leq \varepsilon$, Algorithm \ref{alg:f_ucb} requires at most
\begin{equation}
T = 1+\sum_{i=1}^k g_i^{-1}\left(\max\left[f_i^*-f^* - \frac{\varepsilon}{2}, \frac{\varepsilon}{2}\right]\right)
\end{equation} 
iterations, where $g_i^{-1}(\varepsilon) \triangleq \min \{\tau \mid f_i(x_t) - f_i^* \leq \varepsilon, \; \forall t \geq \tau\}$.
\end{theorem}

\begin{proof}
Assume without loss of generality that $f^* = f_1^*$. 
Let~$k_{i,t}$ denote the total number of times $f_i$ is updated in Algorithm~\ref{alg:f_ucb} \textit{(line ~6)} by iteration $t$.
Then, since in every iteration of Algorithm~\ref{alg:f_ucb} only one function is updated, the number of iterations $T$ can be computed as $T = \sum_{i=1}^K k_{i, T}$.  
Note that from the definition~\ref{def::gkdelta-alg} it follows that for all $k_{1,t}\geq 1$, we have $f_1(x^{1, t}) - g_1(k_{1,t}) \leq f_1^* = f^*$. 
At step $T-1 = \sum_{i=1}^K k_{i,T-1}$, there exists a subproblem $\mathcal{P}_j$ such that the corresponding function~$f_j$ was updated at least $k_j^* = g_j^{-1}(\max(f_j^*-f^* - \frac{\varepsilon}{2}, \frac{\varepsilon}{2}))$ times, i.e. $k_{j, T-1} \geq k_j^*$, since, otherwise, $\sum_{i=1}^k k_{i,T-1} < T-1$.
Note that if $k_{j, t} \geq k_j^*$, then
\begin{equation}
    g_j(k_{j,t}) \leq g_j(k^*_j),
    \label{eq::g_kstar}
\end{equation}
which means the larger the number of iterations, the smaller the gap between $f_j$ and $f^*_j$.%

Let us show that if Algorithm~\ref{alg:f_ucb} selects function $f_j$ \textit{(line ~6)} after making exactly $k_{j, t} =  k_j^*$ corresponding updates, then it reaches the stopping criterion \textit{(lines ~7-9)} and returns $f_j$ such that $f_j^* - f^* \leq \varepsilon$.
According to \textit{lines ~4} in Algorithm~\ref{alg:f_ucb} the following inequality holds:
\begin{equation}
    f_j(x^{j,t}) - g_j(k_{j,t})\leq f_1(x^{1,t}) - g_1(k_{1, t}) \leq f_1^* = f^*.
    \label{eq::hypothesis_step1}
\end{equation}
If $f_j^*-f^*> \varepsilon$, then $\max(f_j^*-f^* - \frac{\varepsilon}{2}, \frac{\varepsilon}{2}) = f_j^*-f^* - \frac{\varepsilon}{2}$ and $g(k_j^*) = f_j^*-f^* - \frac{\varepsilon}{2}$. 
Taking into account the inequality~(\ref{eq::g_kstar}), we have the following inequalities:

\begin{equation}
        f_j(x^{j,t}) - g_j(k_{j,t}) \geq f_j(x^{j,t}) - g_j(k^*_j) 
        = f_j(x^{j,t}) - (f_j^*-f_1^* - \frac{\varepsilon}{2})= f_1^* + (f_j(x^{j,t}) - f_j^*) + \frac{\varepsilon}{2} 
        \geq f_1^* + \frac{\varepsilon}{2}.
    \label{eq::hypothesis_step2}
\end{equation}

Thus, we have a contradiction with inequality~(\ref{eq::hypothesis_step1}), and conclude that $f_j^*-f^* \leq \varepsilon$.
Therefore, $\max(f_j^*-f^* - \frac{\varepsilon}{2}, \frac{\varepsilon}{2}) = \frac{\varepsilon}{2}$, $k_j^* = g_j^{-1}(\frac{\varepsilon}{2})$ and the stopping criterion (\textit{lines 7-9}) holds. So, we demonstrate that $k_{i,t} \leq k_i^*, i = 1,\ldots,K$.
Thus, if the stopping criterion was not reached before iteration $T-1$, then Algorithm~\ref{alg:f_ucb} made exactly $k_i^*$ updates of function~$f_i$ for $i=1,\ldots,K$ by the iteration $T-1$.
After that, in the iteration $T$, Algorithm~\ref{alg:f_ucb} selects and returns $f_j$ such that $k_{j, T-1} = k_j^*$, and the stopping criterion holds.
\end{proof}

\begin{corollary}[BFI polynomial]
If for each base algorithm $g_i(k)=\frac{\beta_i}{k^{r}}$ with $r>0,\ \beta_i>0$.
Then $g_i^{-1}(\varepsilon) = \min\left\{\tau \mid \frac{\beta_i}{\tau^{r}}\le \varepsilon\right\}
= \left\lceil \Bigl(\frac{\beta_i}{\varepsilon}\Bigr)^{1/r}\right\rceil$,
and to achieve \(R_B(T)\le \varepsilon\) Algorithm~\ref{alg:f_ucb} requires at most $T$ iterations, where 
$$
T = 1+\sum_{i=1}^k 
\Biggl\lceil\Biggl(\frac{\beta_i}
{\max\!\Bigl\{\,f_i^*-f^*-\frac{\varepsilon}{2}\,,\,\frac{\varepsilon}{2}\Bigr\}}
\Biggr)^{\!1/r}\Biggr\rceil .
$$
\end{corollary}

\subsection{Stochastic case}
This section presents bounds for regrets $R_O$ and $R_B$, if $g_i(k, \delta)$-bounded base optimizers use inexact oracles.
After that, we consider particular classes of functions $f_i$, select the corresponding $g(k, \delta)$-bounded base optimizers, and derive the final regret bounds.

\paragraph{FMAB}
To prove the bound for $R_O$ regret, we define a \emph{clean event} as follows: 
$$
\mathcal{E}_{\text{clean}} \triangleq \bigcap_{i \in [K], t \in [T]} \left\{ f_i(x_{i,t}) - f_i^* \leq g(k_{i,t}, \delta) \right\}
$$


The probability of this event can be bounded from below. 
Indeed, if the concentration inequality $\mathbb{P}\left[ f_i(x_{i, t}) - f_i^* \geq g(k_{i, t}, \delta) \right] \leq \delta$ holds, then applying the union bound, we get: $\mathbb{P} [\mathcal E_{\text{clean}}] \geq 1 - TK\delta$.

\begin{conjecture}
Functions $f_i$ are bounded above, i.e. $\max_{1 \leq i \leq K}\max_{x_i \in D_i}f_i(x_i) \leq A$.
\label{assumption::1}
\end{conjecture}

With Assumption~\ref{assumption::1}, we get the following regret bound:

\begin{equation}
    \mathbb{E}[R_O(T)] \leq
    \mathbb{E}[R_O(T) \mid {\mathcal{E}_{\text{clean}}}] + KT^2 A \delta.
\label{eq::exp_regret}
\end{equation}

\begin{lemma}
\label{lemma:o_stochastic_regret}
Assume $\mathcal{A}_i$ be $g_i(k, \delta)$-bounded algorithms for the corresponding problem $\min_{x\in \mathcal{D}_i} f_i(x)$ for each $1\leq i \leq K$. 
Then for Algorithm~\ref{alg:f_ucb} the following inequality holds:
\begin{equation}
\mathbb{E}\left[R_O(T)\right] \leq \mathbb{E}\left[ \sum_{i=1}^{K} \sum_{t=1}^{k_{i, T}} g_i(t ,\delta) \Bigg\vert \mathcal{E}_{\text{clean}} \right] + \delta KT^2 \cdot A,
\label{eq:regret_bound}
\end{equation}
where $A = \max_{1 \leq i \leq K} \max_{x_i\in D_i} f_i(x_i)$.
\end{lemma}

\begin{proofidea}
Expected regret~(\ref{eq::exp_regret}) under a complement to a clean event is lower than $\delta KT^2 A$.  
In the conditions of a clean event, for all realizations of the optimization process, the regret is bounded as in Theorem~\ref{lemma:o_regret_bound}.
It only remains to add them up according to~(\ref{eq::ro_fmab_det}). 
Here, expectation is over realization in a clean event.
\hfill $\square$

The complete proof is presented in Appendix~\ref{sec::appendix::fmab_proof}.
\end{proofidea}

\begin{theorem}
\label{thm::stochastic_convergence}
    Let $r > 0$. 
    Assume $\mathcal{A}_i$ ($i=1,\ldots,K$) are $g_i(k, \delta)$-bounded base algorithms with $g_i(k, \delta) = \frac{\beta_i}{k^r}c(\delta)$. set $\delta = \frac{1}{KT^2A}$.
    Then for Algorithm~\ref{alg:f_ucb}, for all $\tau \in \overline{1, T}$ the following regret bounds hold:
    \begin{itemize}
        \item if $r \in (0, 1)$: $R_O(\tau) \leq O\left(\left(\sum_{i=1}^K \beta_i^{\frac{1}{r}}\right)^{r} \tau^{1-r}c(\frac{1}{KT^2A})\right)$;
        \item if $r = 1$: $R_O(\tau) \leq  O\left(\sum_{i=1}^K \beta_i \log (\tau)  c(\frac{1}{KT^2A})\right)$;
        \item if $r > 1$: $R_O(\tau) \leq  O\left((\sum_{i=1}^K \beta_i) c(\frac{1}{KT^2A})\right)$.
    \end{itemize}
\end{theorem}

\noindent
The multiplier $c(\delta)$ typically grows at most polylogarithmically with respect to $\frac{1}{\delta}$, i.e. $c(\delta) = \mathrm{polylog}\left(\tfrac{1}{\delta}\right)$,
as is standard in high-probability convergence bounds for stochastic first-order methods (see, e.g., \cite{lan2020first,sadiev2023high}). 
Consequently, $\delta$ can be chosen to decay polynomially with $T$ to ensure that the additive term $\delta K T^2 A$ remains negligible compared to the main regret term, while only affecting~$c(\delta)$ by a logarithmic factor.


Further, we consider algorithms from~\cite{lan2020first, sadiev2023high} with known convergence rates $g(t, \delta)$.
These algorithms are developed for unconstrained stochastic optimization problems that typically appear in applications~\cite{wallace2005applications}.
Logarithmic factors in the complexity bounds are omitted for clarity.
Convergence guarantees of considered algorithms rely on the following assumptions.

\begin{conjecture}
    \label{assumption::noise_bound}
    The algorithm $\mathcal{A}_i$ has access to the unbiased stochastic first-order oracle, returning $G_i(x, \xi)$. 
    There exists a set~$\mathcal{X}_i \in \mathbb R^d$ and values $\sigma \geq 0, \alpha \in (1, 2]$ such that for all 
    $x \in \mathcal{X}_i:$ 
    $\mathbb{E}_\xi \left[\|  G_i(x, \xi) - \nabla f_i(x)\|^\alpha \right] \leq \sigma^\alpha$.
\end{conjecture}

\begin{conjecture}
\label{assumption::noise_tails}
For any $x$ we have: 
$\mathbb{E} \exp\{\| G_i(x, \xi) - \nabla f_i(x) \|^2 /\sigma_i^2\} \leq 1.$
\end{conjecture}

The resulting $R_O$ regret bounds for the stochastic setup are summarized in Table~\ref{table:stochastic}. 

The proofs for these bounds are based on Theorem~\ref{thm::stochastic_convergence}, available $g(k, \delta)$ in base optimizers, and are presented in Appendix~\ref{sec::appendix::stoch_fmab}.

\begin{remark}
In the stochastic assumptions, $\alpha \in (1,2]$ characterizes the noise distribution. 
For the considered base algorithms, the convergence rate parameter $r$ is an explicit function of $\alpha$: 
for \algname{clipped-SSTM} (convex problems) $r = 1 - \frac{1}{\alpha}$, 
and for \algname{R-clipped-SSTM} (strongly convex problems) $r = 2\left(1 - \frac{1}{\alpha}\right)$.
\end{remark}

\begin{table}[!ht]
    \centering
    \caption{Regret bounds for FMAB problem in the stochastic case. 
    All algorithms require Assumption~\ref{assumption::1}. SSTM algorithms~\cite{sadiev2023high} require Assumption~\ref{assumption::noise_bound} ($\alpha \in (1,2]$). 
    AGD~\cite{lan2020first} requires Assumption~\ref{assumption::noise_bound} ($\alpha=2$) and Assumption~\ref{assumption::noise_tails}.}
    \resizebox{\linewidth}{!}{
    \begin{tabular}{cccccc}
    \toprule
    Function & Base optimizer & $R_O(T)$ \\
       \midrule
    Convex $L$-smooth &  \algname{clipped-SSTM} & $O\left(\max\left[KLR^2, \alpha \sigma R K^{1 - \frac{1}{\alpha}} T^{\frac{1}{\alpha}}\log(AKT)\right]\right)$ \\
    $\mu$-strongly convex, $L$-smooth & \algname{R-clipped-SSTM}  & $O\left(\max \left[K\sqrt{\frac{L}{\mu}} ,{\frac{\sigma^2}{\mu} K^{2\frac{\alpha -1}{\alpha}} T^{\frac{2}{\alpha} -1}}\log(AKT)\right]\right)$ \\
    $\mu$-strongly convex, $M$-Lipschitz & \algname{AGD} & $O\left( \sqrt{KT}\sigma R \log{(AKT)} \right)$ \\
    \bottomrule
    \end{tabular}
    }
    \label{table:stochastic}
\end{table}

\section{Numerical experiments}
\label{sec: Experiments}
To illustrate the performance of the proposed approach, we consider synthetic test cases for convex smooth and nonsmooth functions.
The case of smooth convex functions with inexact first-order oracles is also included in our experimental evaluation.
Finally, we consider the CIFAR100 image classification task and use the \algname{F-LCB} algorithm to automatically identify the best neural network from the given candidate set.
We share the source code in the repository at \url{https://github.com/IAIOnline/FMAB} to reproduce the presented results.

\begin{figure}[!ht]
\centering
\begin{subfigure}[t]{0.31\linewidth}
\includegraphics[width=\textwidth]{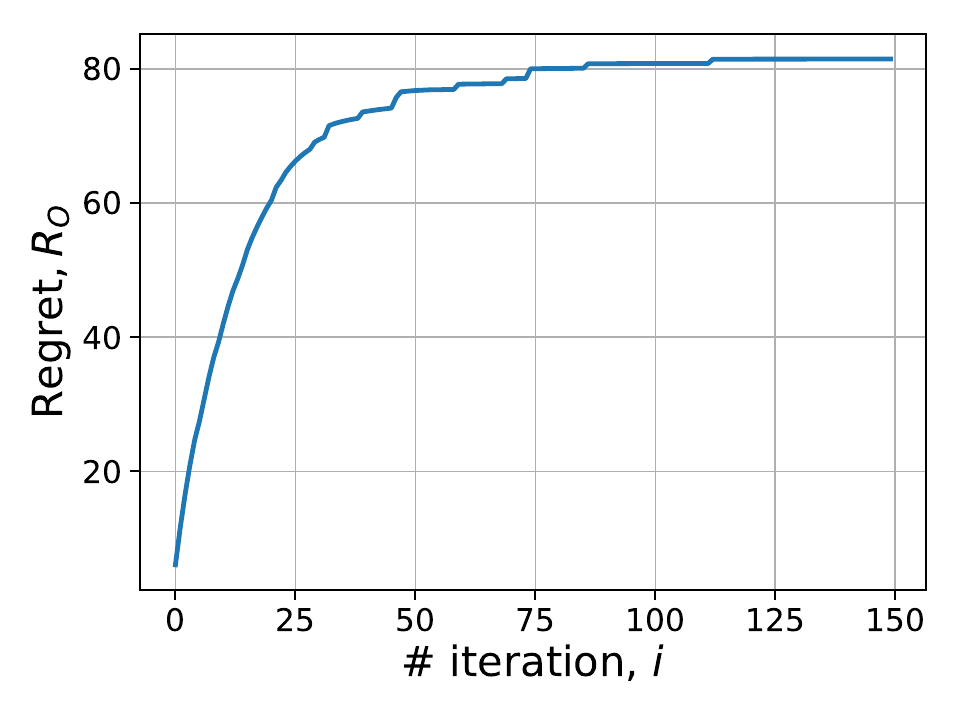}
\caption{Cumulative regret}
\end{subfigure}
~
\begin{subfigure}[t]{0.31\linewidth}
\includegraphics[width=\textwidth]{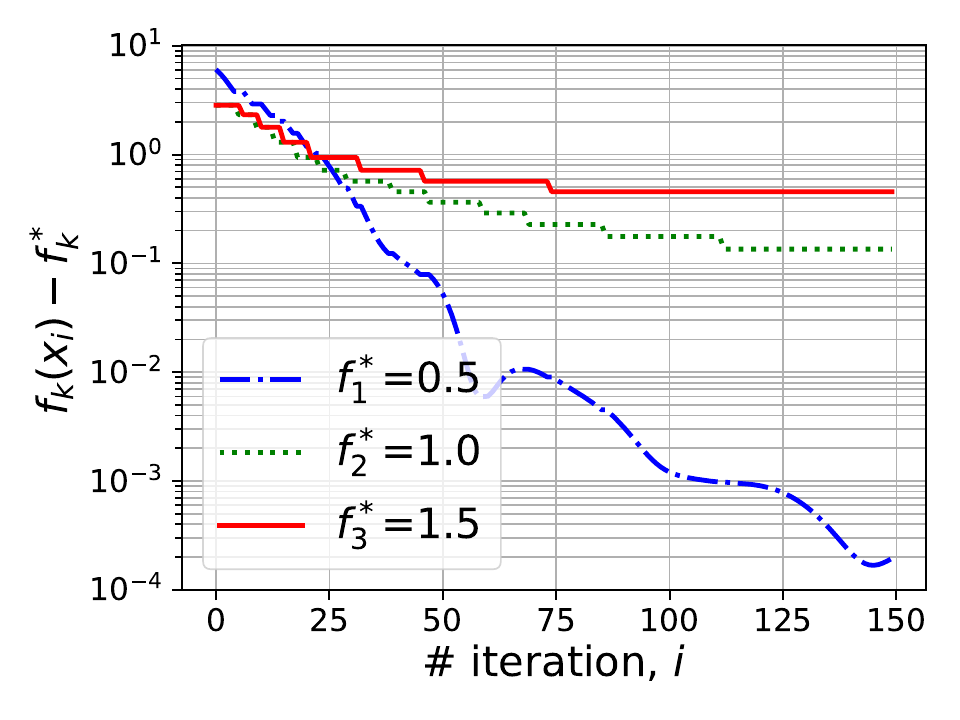}
\caption{Convergence rate}
\end{subfigure}
~
\begin{subfigure}[t]{0.31\linewidth}
\includegraphics[width=\textwidth]{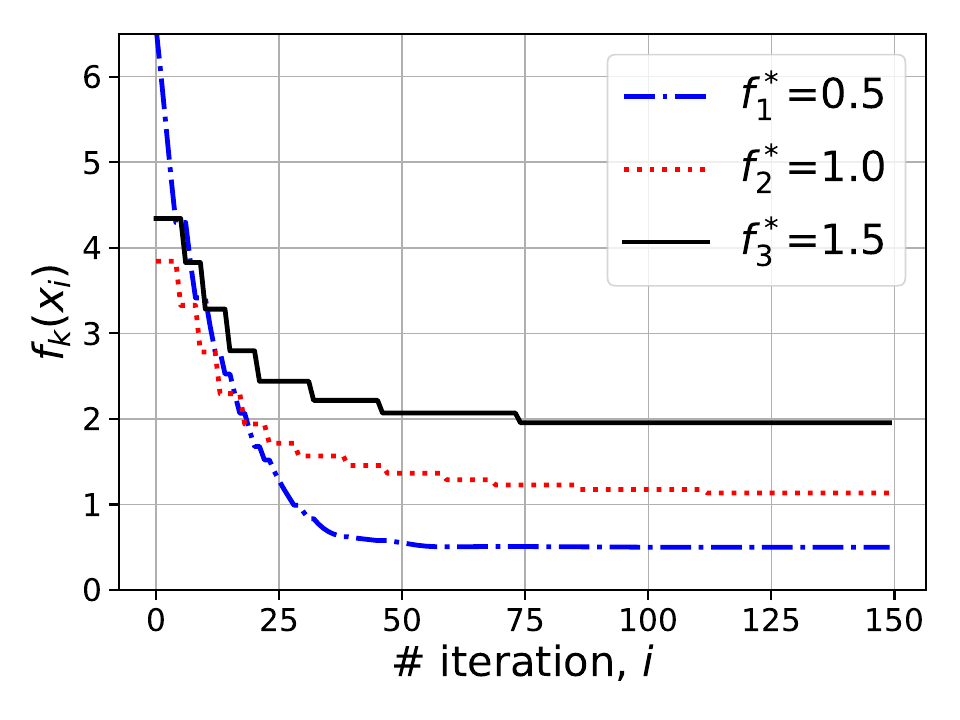}
\caption{Function values}
\label{fig::functions_value_smooth}
\end{subfigure}
\caption{Dependence of cumulative regret (upper left), convergence rate (middle), and function values (right) on iterations of \algname{F-LCB} algorithms for FMAB setup with smooth convex functions~(\ref{eq:exp_smooth}). 
Regret stops increasing after 75 iterations, and \algname{F-LCB} minimizes only $f_1$ with $f_1^* < f_2^*$ and $f_1^* < f_3^*$.}
\label{fig::smooth_convex}

\end{figure}

\subsection{FMAB: smooth convex functions}
We consider the following set of smooth convex functions:
\begin{equation}
\label{eq:exp_smooth}
f_i(\mathbf{x}) = \sqrt{1 + (\mathbf{x}- \mathbf{x}_i^*)^\top \boldsymbol{\Sigma}_i (\mathbf{x} - \mathbf{x}_i^*)} + c_i,
\end{equation}
where $\boldsymbol{\Sigma}_i = \mathrm{diag}(\sigma_1, \ldots, \sigma_d)$ is a diagonal matrix, where $\sigma_i > 0$. 
We set $\sigma_1 = 1$ and $\sigma_i = \exp(-5 \xi)$, where $\xi\sim U[0,1]$ for $i =2, \ldots, d$.
Here, functions are not strongly convex but have a Lipschits gradient with $L_i = \max_{i=1,\ldots, d} \sigma_i $.
We use accelerated gradient descent~\cite{su2016differential,nesterov1983method} as a base optimizer in this setup.
According to~\cite{taylor2017exact}, the function $g$ has the following form for this base optimizer: $g_i(t) = \frac{2L_i\|\mathbf{x}^{0,i}-\mathbf{x}_\star\|^2}{t^2 + 5 t + 6}$ and hense the regret is bounded by constant $O(KLR^2)$.
We generate $K = 3$ instances of functions with $d=20$ and run the algorithm for $T = 200$ steps.
Figure~\ref{fig::smooth_convex} shows that \algname{F-LCB} algorithm automatically selects the function with the smallest optimal value.
After some iterations, it minimizes only this function, while the target variables for other functions are not updated.
The stepwise decreasing of $f_3$ in Figure~\ref{fig::functions_value_smooth} illustrates such behavior.
Thus, \algname{F-LCB} identifies the smooth convex function $f_{1}$ among other similar functions $\{ f_2, f_3 \}$ such that $f_1^* < f^*_2$ and $f_1^* < f^*_3$.

\begin{figure}[ht]
\centering
\begin{subfigure}[t]{0.31\linewidth}
\centering
\includegraphics[width=\textwidth]{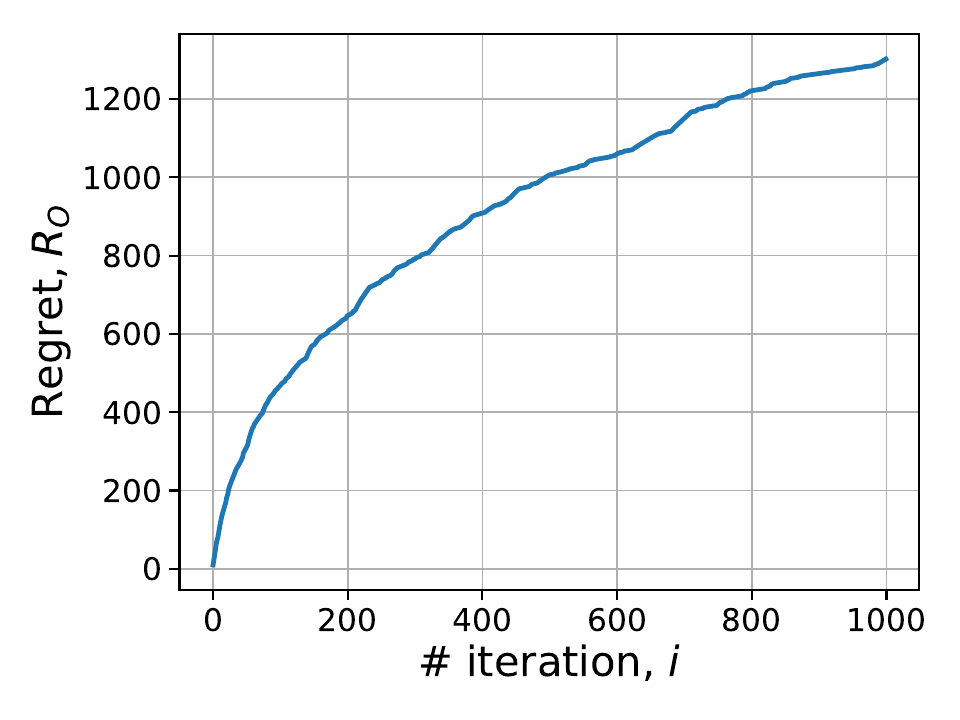}
\caption{Cumulative regret}
\label{fig::cum_regret_nonsmooth}
\end{subfigure}
~
\begin{subfigure}[t]{0.31\linewidth}
\centering
\includegraphics[width=\textwidth]{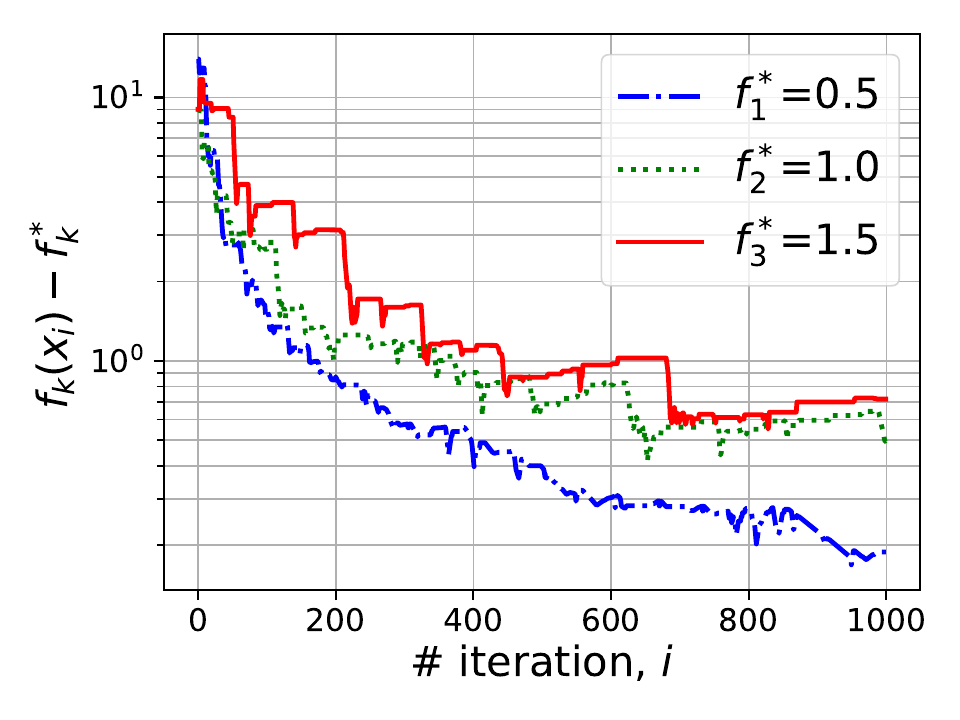}
    \caption{Convergence rate}
\label{fig::convergence_nonsmooth}
\end{subfigure}
~
\begin{subfigure}[t]{0.31\linewidth}
\centering
\includegraphics[width=\textwidth]{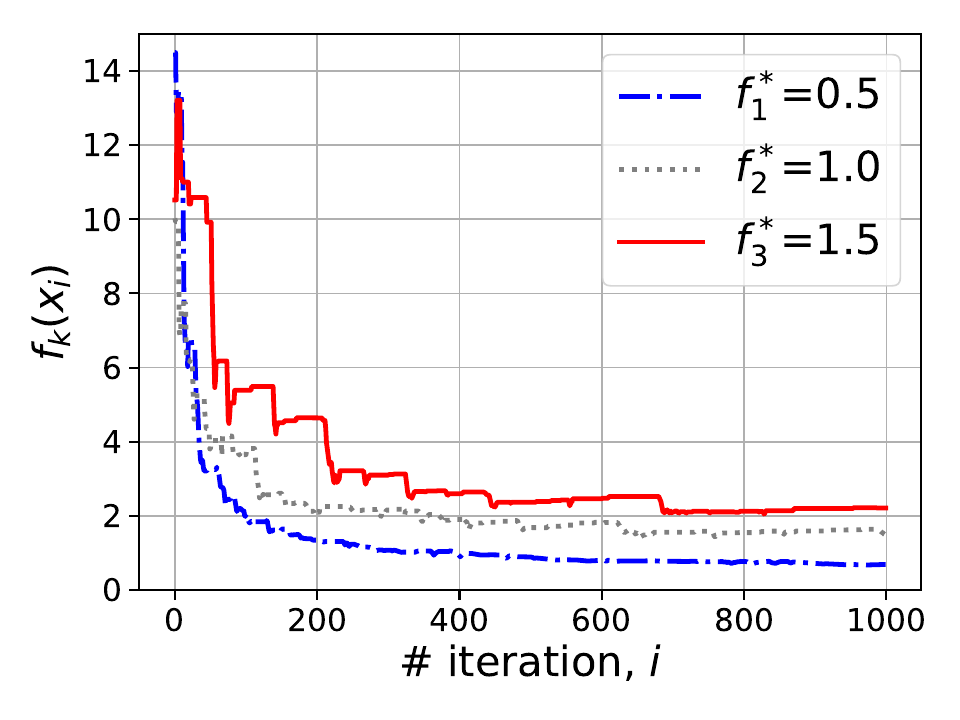}
    \caption{Function values}
\end{subfigure}
\caption{Dependence of cumulative regret (left), convergence rate (middle), and functions values (right) on the number of iterations of \algname{F-LCB} for FMAB setup with nonsmooth convex functions (\ref{exp::nonsmooth_func}). 
The function $f_1$ with the smallest minimal value is identified, and the smallest minimum $f_1^* = 0.5$ is achieved.
Spikes in plots (b) and (c) indicate the switching between the minimization of $\{f_1, f_2, f_3\}$.
}
\label{fig::nonsmooth_convex}

\end{figure}

\subsection{FMAB: nonsmooth convex functions}
To illustrate the performance of our algorithm in the nonsmooth convex setup, we consider piece-wise linear functions with feasible sets $D_i = [-4, 4]^{d}$ and $d=20$:
\begin{equation}
    f_i(\mathbf{x}) = \max_{k=1,\ldots,p} (\mathbf{a}^\top_{ki} \mathbf{x} + b_k^i) + c_i.
    \label{exp::nonsmooth_func}
\end{equation}
We consider $K = 3$ functions and run the algorithm for $T = 1000$ steps.  
We use $p = \{5, 10, 12\}$ linear functions for given minimal values $\{0.5, 1, 1.5\}$ respectively. 
We use the Subgradient Method with Triple Averaging~\cite{nesterov2015quasi} as a base optimizer for such functions.
For this base optimizer, we have $g_i(t) = \frac{M_iR_i}{\sqrt{t}}$. 
Hence, cumulative regret is bounded by $O\left(RM\sqrt{KT}\right)$.
The resulting cumulative regret and function values are presented in Figure~\ref{fig::nonsmooth_convex}.
The convergence of \algname{F-LCB} demonstrates that the minimization process for the target objective function $f_1$ leads to faster convergence to the minimum.



\subsection{FMAB: smooth convex functions with inexact oracle}
To emulate the inexact oracle in the smooth convex setup, we consider functions~(\ref{eq:exp_smooth}) but add noise to the gradients. 
The gradient estimate is computed as $G_i(\mathbf{x}, \xi) = \nabla f_i(\mathbf{x}) + \frac{\sigma_i}{\sqrt{d}}\xi$, where $\xi \sim \mathcal{N}(0, I)$.
We use stochastic accelerated gradient descent as a base optimizer and parameters from proposition~4.5 in~\cite{lan2020first} that give $\mathbb{E}[f(\overline{ \mathbf{x}}_t)] - f^* \leq O(1) \left(2\gamma R + \frac{4\sqrt{2}(M^2 + \sigma^2)}{3\gamma}\right) \frac{1}{\sqrt{t}}$.  
We consider $K = 3$ functions, $T = 1500$ steps, dimension $d=20$ and $\sigma = 2$. 
Figure~\ref{exp:stochastic} shows that although  gradient is inexact, our algorithm finds the best function.
In contract to the deterministic setup, the convergence curves shown in Figure~\ref{fig::convergence_stoch_smooth} are less distinguished.
However, Figure~\ref{fig::fun_vals_stoch_smooth} shows that \algname{F-LCB} pays more attention to the minimization of $f_1$ rather than $f_2$ or $f_3$. 

\begin{figure}[!ht]
\centering
\begin{subfigure}{0.31\linewidth}
\includegraphics[width=\textwidth]{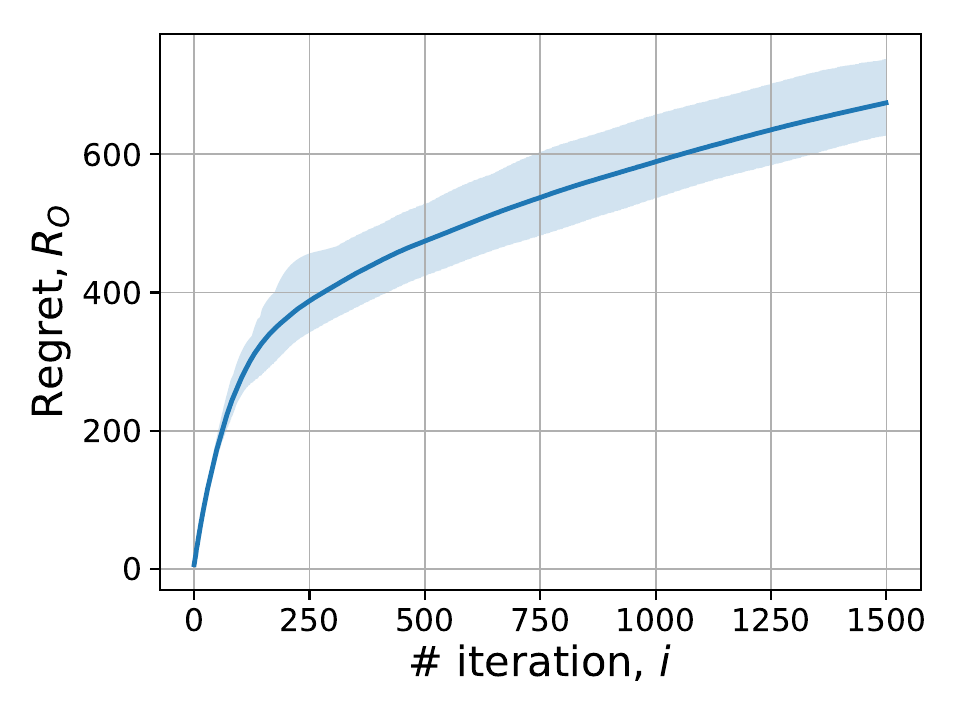}
\caption{Cumulative regret}
\label{fig::cum_regret_stoch_smooth}
\end{subfigure}
~
\begin{subfigure}{0.31\linewidth}
\includegraphics[width=\textwidth]{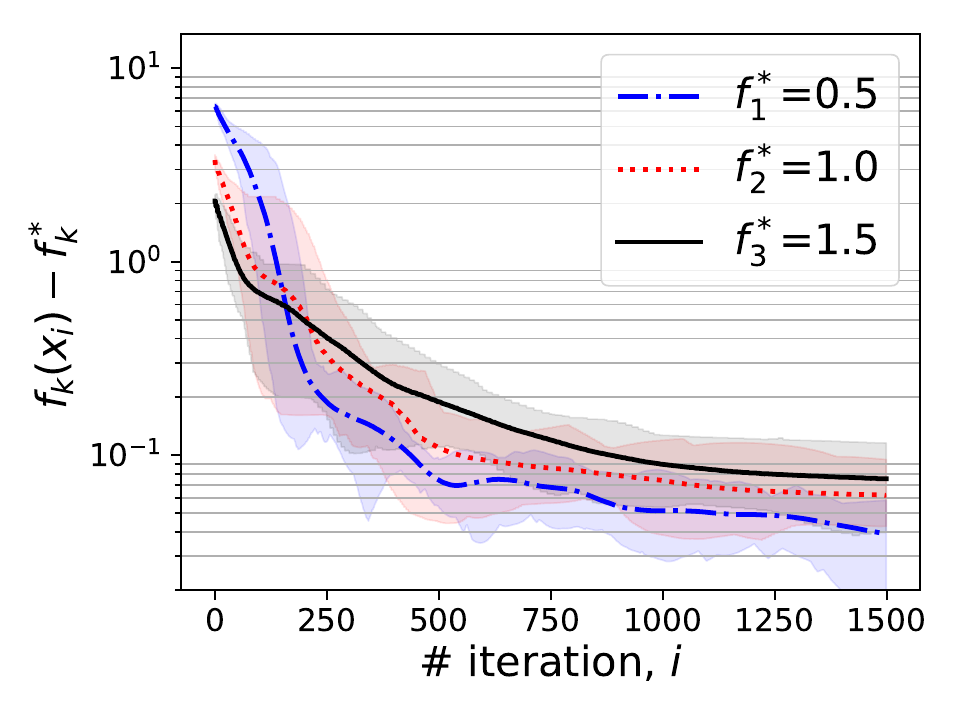}
\caption{Convergence rate}
\label{fig::convergence_stoch_smooth}
\end{subfigure}
~
\begin{subfigure}{0.31\linewidth}
\includegraphics[width=\textwidth]{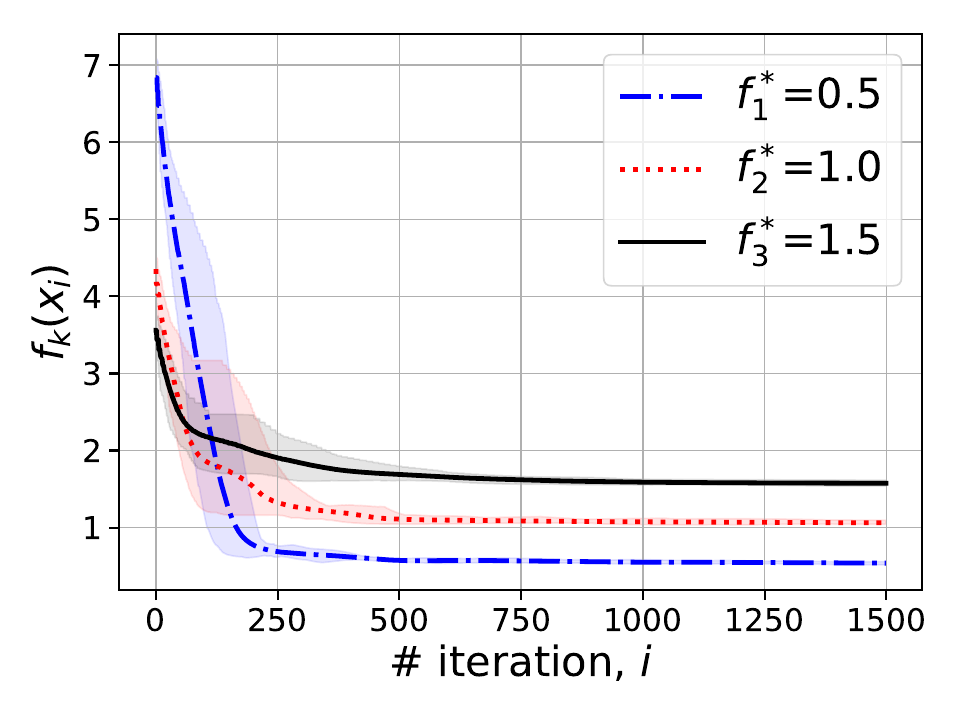}
\caption{Function values}
\label{fig::fun_vals_stoch_smooth}
\end{subfigure}
\caption{Dependence of cumulative regret (left), convergence rate (middle), and functions values (right) on iterations of the \algname{F-LCB} algorithm in the FMAB setup for smooth convex functions with inexact oracles.
The best function is found automatically even if the initial guess is poor.}
\label{exp:stochastic}

\end{figure}

\subsection{BFI: neural network selection}
\label{sec::bfi-cifar100}
\balance
We compare the performance of our algorithm with  \algname{SuccessiveHalving}~\cite{jamieson2016non} (denoted as \algname{SH}) and \algname{Hyperband}~\cite{li2018hyperband} methods.  
We evaluate how well these algorithms can distinguish between neural network architectures with similar performance.
We consider models for the image classification task and train them in the CIFAR-100 dataset~\cite{krizhevsky2009learning} on the single GPU P100.
The selected models have fewer than 5M parameters and are represented as arms in the BFI problem. 
Table~\ref{tab:models} provides a summary of the selected models. 

\paragraph{LCB estimation.}
\label{sec:arm_pull}
In this setup, the pull of the $i$-th arm is 40 updates of the $i$-th model parameters.
After that, the validation loss and accuracy are computed to update the corresponding LCBs.
Since training neural networks is a non-convex problem, convergence guarantees do not directly apply. 
Therefore, to define the function~$g$, we use a heuristic approach inspired by the stochastic optimization theory~\cite{lan2020first}.
In particular, we define $g(t) = \frac{2 \cdot f_i(x^{1, i})}{\sqrt{t}}$, where the nominator estimates the maximum function deviation during the training process, and $x^{1, i}$ is obtained after the first 40 updates.
Since the training process could lead to a large variance in validation losses, we compute the LCBs based on the best validation loss for each model computed before the current step.
More details about the experimental setup and used hyperparameters are in Appendix~\ref{sec::appendix::experiments}.

\begin{table}[!ht]
    \begin{minipage}[c]{0.45\linewidth}
    \caption{Summary of selected neural networks for image clsssification task used to evaluate the \algname{F-LCB}. 
    The models are ranked according to their Top-1 accuracy on the validation set.}
    \resizebox{\linewidth}{!}{
        \begin{tabular}{lcc}
            \toprule
            Model & top-1 acc, \% & $\#$ params, $\cdot 10^6$ \\
            \midrule
            mobilenetv2\_x1\_4 & 75.98 & 4.50 \\
            mobilenetv2\_x1\_0 & 74.20 & 2.35 \\
            shufflenetv2\_x1\_5 & 73.91 & 2.58 \\
            mobilenetv2\_x0\_75 & 73.61 & 1.48 \\
            resnet56 & 72.63 & 0.86 \\
            shufflenetv2\_x1\_0 & 72.39 & 1.36  \\
            resnet44 & 71.63 & 0.67 \\
            mobilenetv2\_x0\_5 & 70.88 & 0.82 \\
            resnet32 & 70.16 & 0.47 \\
            resnet20 & 68.83 & 0.28 \\
            \bottomrule
        \end{tabular}
        }
    \label{tab:models}
\end{minipage}
\hfill
\begin{minipage}[c]{0.45\linewidth}
\captionof{table}{Mean rank of the best model selected by each algorithm. The closer the rank to one, the better the algorithm works.
\algname{F-LCB} algorithm demonstrates strong performance for tested budgets. 
The smallest mean model ranks are bold. \algname{SH} denotes the \algname{SuccessiveHalving} algorithm.
}
\resizebox{\linewidth}{!}{
\begin{tabular}{cccc}
\toprule
Budget, $T$ & \algname{Hyperband} & \algname{SH} & \algname{F-LCB} \\
\midrule
50 & $4.6\pm2.7$ & $2.5\pm0.9$ & $\mathbf{2.2\pm2.7}$ \\
100 & $3.4\pm3.2$ & $2.9\pm0.5$ & $\mathbf{1.1\pm0.3}$ \\
200 & $3.8\pm2.9$ & $1.1\pm0.3$ & $\mathbf{1.1\pm0.3}$ \\
350 & $1.7\pm0.9$ & $1.0\pm0.0$ & $\mathbf{1.0\pm0.0}$ \\
500 & $2.7\pm2.2$ & $1.0\pm0.0$ & $\mathbf{1.0\pm0.0}$ \\
\bottomrule
\end{tabular}
}
\label{exp:alg_compare}
\end{minipage}
\end{table}

\paragraph{Results.}
The experimental comparison of the \algname{F-LCB} algorithm is presented in Table~\ref{exp:alg_compare}, where the mean and standard deviation of the selected model rank are reported.
Each algorithm run is repeated 10 times.
These ranks demonstrate that \algname{F-LCB} can identify the best model using a smaller training budget.
The selected models provide the smallest validation loss for each budget and algorithm. 
In contrast to the competitors, \algname{F-LCB} does not discard models permanently. 
Since \algname{Hyperband} runs \algname{SH} multiple times and splits the budget between them, it shows poor performance.
The parameters for \algname{Hyperband} are chosen so that the total training steps are approximately equal to the given budget.




\section{Conclusion and limitations}
\balance
This work investigates strategies for the functional multi-armed bandit problem (FMAB) and for best function identification (BFI).
We propose a UCB-type algorithm that uses basic optimizers with known large deviation bounds to construct LCB estimates. 
It establishes regret rate guarantees for FMAB and BFI problems that match corresponding lower bounds up to a logarithmic factor.
Extensive experimental evaluation demonstrates the efficiency of the proposed \algname{F-LCB} algorithm.
First, our approach identifies the best functions among smooth and non-smooth functions with base optimizers equipped with deterministic oracles.
Second, we show that \algname{F-LCB} can process the base optimizer with an inexact oracle for smooth objective functions. 
Finally, we compare \algname{F-LCB} with SuccessiveHalving and Hyperband to identify the best neural network for the given task.
Our method outperforms competitors if a small computing budget is available.
In settings with moderate or large budgets, \algname{F-LCB} performs similarly to competitors and identifies the best model for the task. 



\paragraph{Limitations.} 
Our approach requires the function $g(k, \delta)$ for the base optimizer, since $g(k, \delta)$ is directly used for LCB computation.
Therefore, one must know the objective's functional class and convergence rates for the base algorithm in advance. One also needs to observe objective values or their approximations.
This requirement is satisfied for most classical ML models and corresponding training algorithms. 
However, it often fails for NN-based models since there are no convergence rates (in terms of objective value) for algorithms used for NN training aligned with definition~\ref{def::gkdelta-alg}.

\section*{Acknowledgements}
This work was supported by the The Ministry of Economic Development of the Russian Federation in accordance with the subsidy agreement (agreement identifier 000000C313925P4H0002; grant No 139-15-2025-012).


\bibliographystyle{unsrt}
\bibliography{lib}

\newpage
\appendix

\tableofcontents

\section{Clarification on oracle notation}
\label{oracles}
To the best of our knowledge, oracles were first introduced in the seminal work \cite{nemirovskij1983problem}. 
The oracle serves as an information bridge between the problem and the method. The main idea behind oracles is that they allow us to classify optimization methods based on what information from the problem they use. For example, methods that use objective values only are usually referred to as zero-order methods, because they use only zero-order oracle (in \cite{nemirovskij1983problem} notation). Gradient descent, accelerated gradient descent (with GD with momentum) or ADAM are first-order methods, and the Newton method is a second-order method. In short, order is decided by the order of the derivative that the method uses when iterates. There also exist methods with other types of oracles (such as separation oracle). 

The main reason is to factorize the complexity of any algorithm in the following form: $C_{\mathcal{A}} (\varepsilon, \mathcal{P}) = I_{\mathcal{A}} (\varepsilon, \mathcal{P})\cdot O_{\mathcal{A}}(\mathcal{P})    $, where $C_{\mathcal{A}} (\varepsilon)$ is the overall complexity (time or atomic operations) needed for the algorithm $\mathcal{A}$ to achieve $\varepsilon$-optimal point for the problem $\mathcal{P}$; $I_{\mathcal{A}} (\varepsilon, \mathcal{P})$ is the number of iterations the algorithm $\mathcal{A}$ would require, $O_{\mathcal{A}}(\mathcal{P})$ is the oracle complexity for one iteration of the algorithm $\mathcal{A}$. This captures the idea that methods that use more information about the problem could perform better, but the price is that information is harder to get or compute. For example, gradient-type methods are in general more efficient in terms of the number of iterations needed to find an approximate solution than zero-order methods, but each iteration is much more expensive because the gradient must be computed and stored. 

An immediate take-off is that this notation allows us to reason about methods' optimality. If two methods use the same oracle type, but one of them requires fewer iterations, then it is better. This leads to the so-called oracle complexity and lower bound for classes of methods defined by the oracle type they use. This is a standard approach to complexity in nonlinear optimization: any lower and upper bounds are provided for a pair of problem class and method class, and the method class is defined by oracle type. We use this notation to achieve lower bounds in the next sections.

To make it more familiar, here is an example of how the gradient descent method could be described via oracle notation:

For the problem
$$
\min_x f(x)
$$
first-order oracle $\mathcal{O}(x)$ contains objective value at the point $x_t$ and its gradient $\mathcal{O}(x_t) = \{f(x_t), \nabla f(x_t)\}$. 

Hence, when gradient descent is running, at each iteration $t$ it first calls a routine that computes the gradient at the current query point $x_t$, i.e., calls for the oracle $\mathcal{O}(x_t)$, then takes the gradient and makes an update step 
$$
x_{t+1} = \mathcal{A}(x_1, \mathcal{O}(x_1), \dots, x_t, \mathcal{O}(x_t)) = x_t - \mu_t \nabla f(x_t).
$$
On the other hand, the Newton method could not be defined via first-order oracles, it require second-order derivative.

\section{Additional proofs for regret rates.}
\subsection{Stochastic FMAB}
\label{sec::appendix::fmab_proof}

\textbf{Theorem 3} 
Assume $\mathcal{A}_i$ be $g_i(k, \delta)$-bounded algorithms for the corresponding problem $\min_{x\in \mathcal{D}_i} f_i(x)$ for each $1\leq i \leq K$. 
Then for Algorithm~\ref{alg:f_ucb} the following inequality holds:
\begin{equation}
\mathbb{E}\left[R_O(T)\right] \leq \mathbb{E}\left[ \sum_{i=1}^{K} \sum_{t=1}^{k_{i, T}} g_i(t ,\delta) \Bigg\vert \mathcal{E}_{\text{clean}} \right] + \delta KT^2 \cdot A,
\label{eq:regret_bound:app}
\end{equation}
where $A = \max_{1 \leq i \leq K} \max_{x_i\in D_i} f_i(x_i)$.

\begin{proof}
Let's consider clean event $\mathcal{E}_{clean}$~\ref{def::clean_event}. 
Then, according to inequality~(\ref{eq::exp_regret}), expected regret~(\ref{eq:regret_bound:app}) under a complement to a clean event is lower than $\delta KT^2 A$.   
In the conditions of an event from a set of clean events, for all realizations of the optimization process, the concentration inequalities hold for each arm at all times. So, we may bound regret as in Theorem~\ref{lemma:o_regret_bound}. In the following, the proof is presented for completeness. 
Let $i_t$ be the arm selected at time $t$. 
Then, its LCB value is the smallest one among the arms. 
That is, for all $j$:
\begin{equation}
LCB_{i_t}(k_{i_t, t}) =  f_{i_t}(x^{i_t, k_{i_t, t}}) - g_{i_t}( k_{i_t, t}, \delta) \leq f_j(x^{j, k_{j, t}}) - g_j(k_{j, t}, \delta) \leq f_j^*.
\end{equation}

In particular, this holds w.r.t. the best arm, yielding an estimate of the per-step regret:
\begin{equation}
    f_{i_t}(x^{i_t, k_{i_t, t}}) -g_{i_t}(k_{i_t, t}, \delta) \leq f^* \Rightarrow f_{i_t}(x^{i_t, k_{i_t, t}}) - f^* \overset{(\star)}{\leq} g_{i_t}(k_{i_t, t}, \delta).
\end{equation}

Summing up inequality ($\star$) for $t=1, \ldots, \tau$ we get regret rate:
\begin{equation}
    \sum^{\tau}_{t = 1}f_{i_t}(x^{i_t, k_{i_t}}) - f^* \overset{(1)}{\leq} \sum^{\tau}_{t = 1} g_{i_t}(k_{i_t}, \delta) \overset{(2)}{=} \sum_{i=1}^{K} \sum_{t= 1}^{k_{i,\tau}} g_i(t, \delta).
\end{equation}
Equality (2) is obtained by grouping terms over arms.
\end{proof}

\subsection{Deterministic case for particular class functions}
\label{sec::proof_det_table}
In the following, we present the proof of regrets \ref{tab:summary} $R_O$ and $R_B$ rates for BFI and FMAB problems in the deterministic setup. $g(k, \delta)$ bounds and algorithms for deterministic case taken from \cite{bubeck2015convex}.

Denote by $\lesssim$ inequality up to numerical factors. For \textbf{\textit{Convex $M$-Lipschitz functions}} we use the Projected Gradient Descent (PGD) algorithm, which is $g(k) = \frac{RM}{\sqrt{k}}$ bounded. 

\textit{BFI}. According to Theorem~\ref{th:bfi_det_bound}, to achieve regret $R_B$ bounded by $\varepsilon$, Algorithm~\ref{alg:f_ucb} requires at most\\ 
\begin{equation*}
T = 1 + \sum\limits_{i=1}^K g_i^{-1}\left(\max\left[f_i^*-f^* - \frac{\varepsilon}{2}, \frac{\varepsilon}{2}\right]\right)
\stackrel{g_i^{-1}(x) = \frac{M_i^2R_i^2}{x^2}}{\lesssim }\sum\limits_{i=1}^K\Bigl \lceil\frac{M_i^2R_i^2}{\max(f_i^*-f^*-\frac{\varepsilon}{2},\frac{\varepsilon}{2})^2}\Bigr \rceil
\end{equation*}

steps. For other cases, the proof follows a similar approach, with the corresponding $g(k)$.

\textit{FMAB}. To bound regret $R_O(T)$, first use Theorem~\ref{lemma:o_regret_bound}. First bound the internal sum $\sum_{t =1}^{k_i}\frac{1} {\sqrt{t}} \lesssim \sqrt{k_i}$. 
Then regret is bounded by using Hölder's inequality:
\begin{equation}
    R_O(T) \lesssim\sum_{i = 1}^K M_iR_i\sqrt{k_i} \leq  O\left( \sqrt{T\cdot \sum_{i=1}^KM_i^2R_i^2 } \right).
\end{equation}

\textbf{\textit{Convex $L$-smooth functions}}. Accelerated Gradient Descent is a $g(k) = \frac{LR^2}{k^2}$ bounded algorithm. 

\textit{BFI}. 
\begin{equation*}
T = 1 + \sum\limits_{i=1}^K g_i^{-1}\left(\max\left[f_i^*-f^* - \frac{\varepsilon}{2}, \frac{\varepsilon}{2}\right]\right) 
\stackrel{g_i^{-1}(x) = \sqrt{\frac{L_i R_i^2}{x}}}{\lesssim }\sum\limits_{i=1}^K\left\lceil \sqrt{\frac{L_iR_i^2} {\max(f_i^*-f^*-\frac{\varepsilon}{2}, \frac{\varepsilon}{2})}} \; \right\rceil.
\end{equation*}

\textit{FMAB}. Again summing up the intermal sum we get $\sum_{t =1}^{k_i}\frac{1} {t^2} \lesssim \frac{1}{k_i}$ we obtain regret bounds:
\begin{equation}
    R_O(T) \lesssim\sum_{i = 1}^K L_iR_i^2\frac{1}{k_i} \leq \min\left(TLR^2, \sum_i^KL_iR_i^2\right).    
    \end{equation}

\textbf{\textit{$\mu$-strongly convex $M$ -Lipschitz functions}}. 
Here we use PGD, which is $g(k) = \frac{M^2}{\mu k}$ bounded algorithm. 

\textit{BFI}.
\begin{equation*}
T = 1 + \sum\limits_{i=1}^K g_i^{-1}\left(\max\left[f_i^*-f^* - \frac{\varepsilon}{2}, \frac{\varepsilon}{2}\right]\right) 
\stackrel{g_i^{-1}(x) = \frac{M_i^2}{\mu_i x}}{\lesssim }\sum\limits_{i=1}^K\Bigl \lceil\frac{M_i^2}{\mu_i \max(f_i^*-f^*-\frac{\varepsilon}{2}, \frac{\varepsilon}{2})}\Bigr \rceil.
\end{equation*}

\textit{FMAB.} We obtain upper bound for internal sum as $\sum_{t =1}^{k_i}\frac{1} {t} \lesssim \log {k_i}$. 
Hence, regret is bounded in following ways:
\begin{align*}
    R_O(T) \lesssim\sum_{i = 1}^K \frac{M_i^2}{\mu_i}\log k_i \leq  \left(\sum_{i=1}^K \frac{M_i^2}{\mu_i}\right) \log T\\
    R_O(T) \lesssim\sum_{i = 1}^K \frac{M_i^2}{\mu_i}\log k_i \leq \left(\frac{M^2}{\mu} K\log{\frac{T}{K}}\right).
    \end{align*}
    
For \textbf{\textit{ $\mu$-strongly convex $L$-smooth functions}} we use AGD, which is $g(k) = R^2 \exp\{-\frac{k}{\sqrt{\kappa}}\}$ bounded. 

\textit{BFI}.

\begin{equation*}
T = 1 + \sum\limits_{i=1}^K g_i^{-1}\left(\max\left[f_i^*-f^* - \frac{\varepsilon}{2}, \frac{\varepsilon}{2}\right]\right) 
\stackrel{g_i^{-1}(x) = \sqrt{\kappa} \log \frac{R_i^2}{x}}{\lesssim } \sum\limits_{i=1}^K\Bigl \lceil\sqrt{\kappa}\log\left(\frac{R_i^2}{\max(f_i^*-f^*-\frac{\varepsilon}{2}, \frac{\varepsilon}{2})}\right)\Bigr \rceil.
\end{equation*}

\textit{FMAB}. Using the fact, that internal sums are geometric progressions we bound them with $\frac{\exp \left( -\frac{1}{\sqrt{\kappa_i}}\right)}{\exp \left( -\frac{1}{\sqrt{\kappa_i}}\right) - 1} = \frac{1}{\exp \left( \frac{1}{\sqrt{\kappa_i}}\right) - 1}$. Then the bound for $R_O(T)$ follows:

\begin{equation}
    R_O(T) \lesssim\sum_{i = 1}^K\frac{R_i^2}{\exp \left( \frac{1}{\sqrt{\kappa_i}}\right) - 1} .
\end{equation}

\subsection{Stochastic case for particular class functions}
\label{sec::appendix::stoch_fmab}

\textit{$L$-smooth} functions. 
Under assumption~\ref{assumption::noise_bound} and from Theorem 3.2 in~\cite{sadiev2023high} follows that algorithm \algname{clipped-SSTM} 
\begin{equation*}
g(k,\delta) = \max\left({\frac{LR^2}{k^2}}, \frac{\sigma R}{k^{1- \frac{1}{\alpha}}}\right )\log\frac{1}{\delta}
\end{equation*}
bounded. 
Based on this convergence rate, we obtain a regret bound $R_O(T) \leq O(\max(KLR^2, \alpha \sigma R K^{1 - \frac{1}{\alpha}} T^{\frac{1}{\alpha}}\log(AKT))$, where $\alpha \in (1,2]$.

\textit{$\mu$-strongly convex, $L$-smooth} functions. 
Under assumption~\ref{assumption::noise_bound}, we use \algname{R-clipped-SSTM} algorithm. From Theorem 3.2 in~\cite{sadiev2023high} it is $g(k, \delta) = \max\left(\exp\{-\frac{k}{\sqrt{\kappa}}\}, \frac{\sigma^2}{\mu}k^{-\frac{2(\alpha-1)}{\alpha}} \right)\log(\frac{1}{\delta})$ bounfrf. We obtain $R_O(T) \leq O(\max K\sqrt{\frac{L}{\mu}} ,\left({\frac{\sigma^2}{\mu} K^{2\frac{\alpha -1}{\alpha}} T^{\frac{2}{\alpha} -1}}\log(AKT)\right))$ from Theorem 3.2 in~\cite{sadiev2023high}.

\textit{$M$-Lipschitz, $\mu$-strongly convex, nonsmooth} functions.
Since functions may not be smooth, we need more restrictive assumption~\ref{assumption::noise_bound} ($\alpha=2$) and assumption~~\ref{assumption::noise_tails}.
The base optimizer for this setup is the stochastic AGD algorithm from section 4 in~\cite{lan2020first}. 
From proposition 4.6, we get:
\begin{equation*}
    \mathbb P\left[ f(x^t) - f^* \geq \frac{4LR}{t(t+1)} + \frac{4(M^2 + \sigma^2)}{\mu(t + 1)} + \left(\frac{2\sigma R}{\sqrt{3t}} + \frac{4\sigma^2}{\mu t}\right) \log \left(\frac{1}{\delta}\right)\right] \leq 2\delta.
\end{equation*}
Hence for $\delta = \frac{1}{AKT^{2}}$ the following regret bound holds:
\begin{equation}
R_O(T) \leq O\left( \sqrt{KT}\sigma R \log{(AKT)} \right).
\end{equation}

\section{Experimental evaluation details}
\label{sec::appendix::experiments}

\subsection{CIFAR-100 dataset}

We take the models from the repository \url{https://anonymous.4open.science/r/FMAB/README.md} for the experiment with the CIFAR-100 dataset.
The complete description of the selected models is presented in Table~\ref{tab:models_complete}.

\begin{table}[!ht]
        \caption{Summary of selected neural network models for the CIFAR-100 experiment. The models are ranked according to their Top-1 accuracy on the validation set.}
    \centering
    \resizebox{\linewidth}{!}{
        \begin{tabular}{rlrrrr}
            \toprule
            Rank & Model & Top-1 Acc.(\%) & Top-5 Acc.(\%) & $\#$Params.(M) & $\#$MAdds(M) \\
            \midrule
            1 & mobilenetv2\_x1\_4 & 75.98 & 93.44 & 4.50 & 170.23 \\
            2 & mobilenetv2\_x1\_0 & 74.20 & 92.82 & 2.35 & 88.09 \\
            3 & shufflenetv2\_x1\_5 & 73.91 & 92.13 & 2.58 & 94.35 \\
            4 & mobilenetv2\_x0\_75 & 73.61 & 92.61 & 1.48 & 59.43 \\
            5 & resnet56 & 72.63 & 91.94 & 0.86 & 125.75 \\
            6 & shufflenetv2\_x1\_0 & 72.39 & 91.46 & 1.36 & 45.09 \\
            7 & resnet44 & 71.63 & 91.58 & 0.67 & 97.44 \\
            8 & mobilenetv2\_x0\_5 & 70.88 & 91.72 & 0.82 & 28.08 \\
            9 & resnet32 & 70.16 & 90.89 & 0.47 & 69.13 \\
            10 & resnet20 & 68.83 & 91.01 & 0.28 & 40.82 \\
            11 & shufflenetv2\_x0\_5 & 67.82 & 89.93 & 0.44 & 10.99 \\
            \bottomrule
        \end{tabular}
        }
    \label{tab:models_complete}
\end{table}

\paragraph{Training hyperparameters.}
We use a batch size of 64 and the Adam optimizer~\cite{kingma2014adam} with a learning rate of $10^{-4}$. 
Each arm pull consists of $40$ optimization steps. 
The budget for \algname{F-LCB} is $T = 200$.
Each model was trained for 100 epochs. 
We split the entire dataset into three parts: train (45000 samples), validation (5000 samples), and test (10000 samples) sets.

\subsection{BFI: neural networks for solving image classification problem}
This section considers an application of the proposed \algname{F-LCB} algorithm to the training neural network models on the CIFAR10 dataset~\cite{krizhevsky2009learning}. 
The main challenge in applying neural networks to solve practical problems is the choice of the proper architecture for the given memory budget of available GPUs.
In our experiment, we consider five different models that have almost the same number of parameters ($\approx 11.5$M) but structure them in a completely different manner. 
In particular, we select the following models: \texttt{ResNet18}~\cite{he2016deep}, convolution neural network similar to VGG~\cite{simonyan2014very}, which we refer as \texttt{VGG}, shallow MLP with only two large linear layers, which we denote as \texttt{ShallowMLP} and deep MLP with eight linear layers of moderate size, which we denote as \texttt{DeepMLP}.
The number of trainable parameters in the considered models is presented in Table~\ref{exp::model_parameters}.
Note that \texttt{DeepMLPNorm} has the same structure as \texttt{DeepMLP} except the intermediate normalization layers added to improve stability of the gradient propagation.
More detailed description of the considered model can be find in the source code.
We expect that our \algname{F-LCB} algorithm identifies the best neural network architecture to solve image classification task.

\begin{table}[!ht]
\centering
\caption{Number of trainable parameters in the considered candidate neural networks.}
\begin{tabular}{lc}
\toprule
 Model name      & \# parameters \\
\midrule
 \texttt{ShallowMLP}  & $12.64$M     \\
 \texttt{DeepMLP}     & $11.55$M    \\
 \texttt{DeepMLPNorm} & $11.57$M   \\
 \texttt{VGG}     & $11.47$M   \\
 \texttt{ResNet18}    & $11.17$M     \\
\bottomrule
\end{tabular}
\label{exp::model_parameters}
\end{table}

\paragraph{Training hyperparameters.}
We use a batch size of 64 and the Adam optimizer~\cite{kingma2014adam} with a learning rate of $10^{-4}$. 
Each arm pull consists of $40$ optimization steps. 
The budget for \algname{F-LCB} is $T = 200$.
We split the entire dataset into three parts: train (45000 samples), validation (5000 samples), and test (10000 samples) sets.
The LCB estimation procedure is the same as discussed in the main text in Section~\ref{sec::bfi-cifar100}.

\paragraph{Best model identification results.}

We run the proposed algorithm 10 times and report the averaged test loss and test accuracy of the models in Figure~\ref{exp:cv_models}.
In addition, we show the $[0.1, 0.9]$ quantile confidence interval via the shaded area.
Figure~\ref{exp:cv_models} demonstrates that the \algname{F-LCB} algorithm confidently identifies the two best models.
Moreover, the identified models are \texttt{ResNet18}, and \texttt{VGG} which are convolutional neural networks and therefore more efficient in solving the considered image classification problem~\cite{goodfellow2016deep}.
In addition, note that \texttt{ResNet18} dominates \texttt{VGG} during the first 150 iterations of our algorithm. 
Thus, we confirm that the proposed framework is relevant for the online identification of the best neural network for a particular dataset.

We compare our approach with the na\"ive baseline based on the early stopping technique.
The early stopping technique stops training if the model's test accuracy does not improve by more than $1\%$ over the consequent $5$ epochs.
We subsequently train the considered neural networks with the early stopping technique and identify the best model based on the best test accuracy.
This baseline identifies the \texttt{ResNet18} as the best model, too.
However, it requires $19 \pm 0.8$ minutes averaged for five runs.
At the same time, the proposed \algname{F-LCB} algorithm identifies the best model only for $9 \pm 0.5$ minutes averaged over the same number of runs for 200 pulls, which is enough in our case to estimate best arms.
In addition, beyond accelerating model selection, our approach monitors the model training online and provides real-time insights into confidence intervals for models' learning curves.

\begin{figure}[!ht]
\centering
\includegraphics[width=\linewidth]{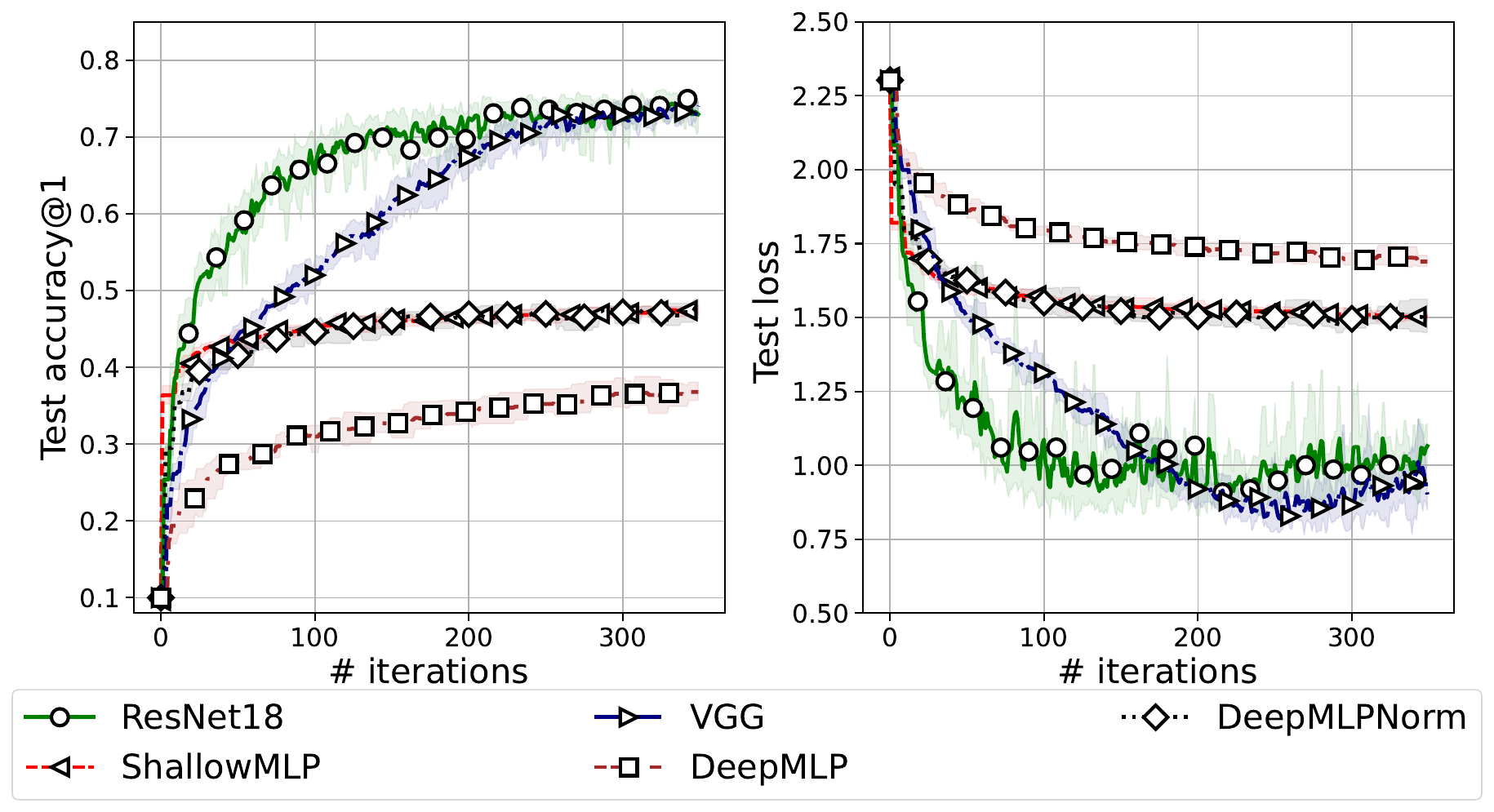}
\caption{Test accuracy and loss of the considered models. 
The shaded area shows $[0.1; 0.9]$ quantiles after running for 10 times. \texttt{ResNet18} and \texttt{VGG} models are the most efficient for solving the considered task that coincides with the previous studies. 
}
\label{exp:cv_models}

\end{figure}

\section{Reduction from MAB to FMAB}

Let us consider a stochastic multi-armed problem setup with time horizon $T$ and $K$ arms, each arm $i$ is equipped with an unknown distribution $\mathcal{D}_i$ with an expected value $\mu_i>0$. At each time step $t$ the agent chooses arm $i_t$ and observes reward $\mu_{i_t}+\xi_{t}$ sampled from $\mathcal{D}_{i_t}$ (noise $\xi_t$, by construction, is unbiased $\mathrm{E}[\xi_t] = 0$). How can we model this problem as FMAB?

Consider FMAB setting with $f_i(x) = \frac{1}{2}(x-\mu_i)^2-\frac{\mu_i^2}{2}$. Here $x$ serves as an estimation of expected value $\mu_i$.

Note that this setting achieves a few good properties:
\begin{itemize}
    \item Optimal objective values represent arms: $f_i^* = -\frac{\mu_i^2}{2}$,
    \item Objective value $f_i(x_t) = \frac{x^2_t}{2} - x_t \mu_i$ can be estimated via tractable $\hat{f}_i(x_t) = \frac{x_t^2}{2} - x_t(\mu_i+\xi_t)$ ($\mu_i+\xi_t$ is a sampled reward from $\mathcal{D}_i$). Estimation is unbiased $\mathrm{E}[\hat{f}_i(x) ] = f_i(x)$.
    \item First order stochastic oracle $\mathcal{O}_i(x_t) = x_t - \mu_i - \xi_t$ is available and provide unbiased gradient estimation $\mathrm{E}[\mathcal{O}_i(x_t)] = x_t - \mu_i = \nabla f_i(x)$.
\end{itemize}

We evaluate reduction method on a heavy-tailed stochastic multi-armed bandit problem (see the full setup in \cite{dorn2024fast}). The instance consists of $K = 10$ arms indexed by  $i \in \{0, 1,\dots,9\}$. The true mean reward of arm $i$ is $\mu_i = \frac{i}{10}$, so the arms are ordered from the worst $\mu_0 = 0$ to the best $\mu_9 = 0.9$. At each pull the observed reward $\mu_{i_t}+\xi_{t}$ is corrupted by heavy tailed standard Cauchy noise $\xi_t$ with the CDF $p(x) = \frac{1}{\pi(1 + x^2)}$. Our algorithm is a slightly modified versionof \algname{SGD-UCB-SMoM} from \cite{dorn2024fast}. The only change is the Upper Confidence bound computation way. According to analysis above and Theorem 8 from \cite{dorn2024fast} we set $UCB(i, n_{i, t}, \delta) := [\frac{x_{i,t}^2}{2} - x_{i,t}(\mu_i+\xi_t)] + \frac{C}{\sqrt{n_{i, t}}}$, where $n_{i, t}$ is the number of the times arm i has been selected, $x_{i,t}$ is a parameters on arm $i$ at time $t$ and C is a constant, that determines the convergence speed of optimizer on arm. Compared algorithms are \algname{SGD-UCB-SMoM} -- the original algorithm from \cite{dorn2024fast}. \algname{SGD-UCB-SMoM-\textbf{F}} is our \textbf{F}unctional version. \algname{RUCB-Median} -- the Robust UCB algorithm of \cite{bubeck2013bandits} that uses median of means to estimate rewards. Figure  \ref{fig::mab_experiment} displays the cumulative regret and mean regret for the three methods. Algorithms are runned for $T = 25000$ steps, with the curves representing the mean of 30 trials. The shaded area around the curves indicates the standard deviation. As expected, the functional version (\algname{SGD‑UCB‑SMoM‑F})  achieves a regret rates comparable to that of \algname{SGD‑UCB‑SMoM}.

\begin{figure}[!ht]
\centering
\begin{subfigure}{0.31\linewidth}
\includegraphics[width=\textwidth]{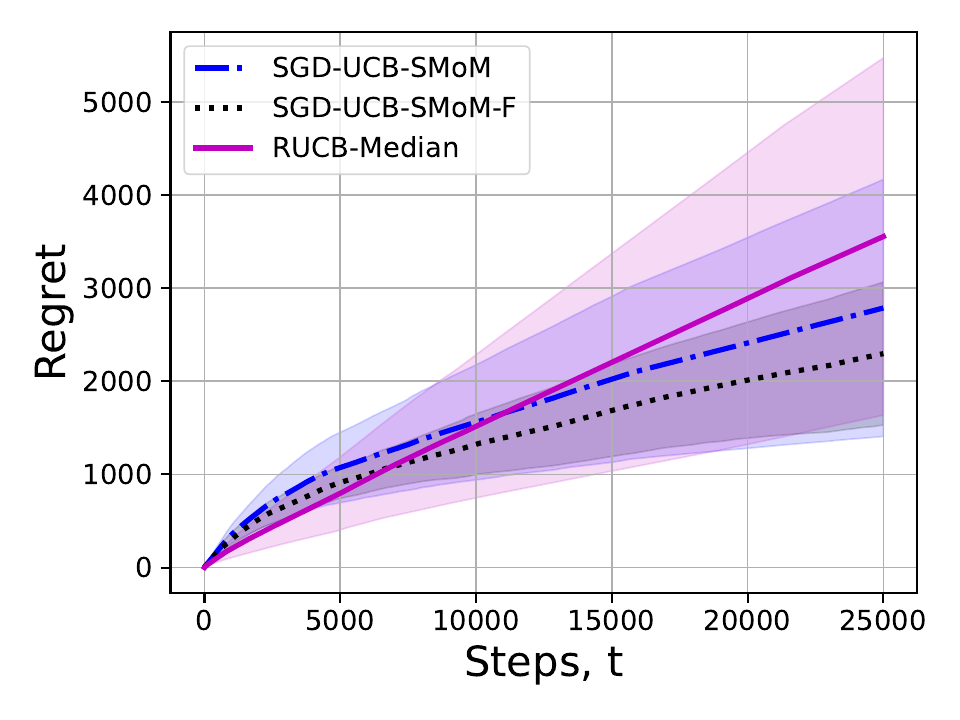}
\caption{Cumulative regret}
\end{subfigure}
~
\begin{subfigure}{0.31\linewidth}
\includegraphics[width=\textwidth]{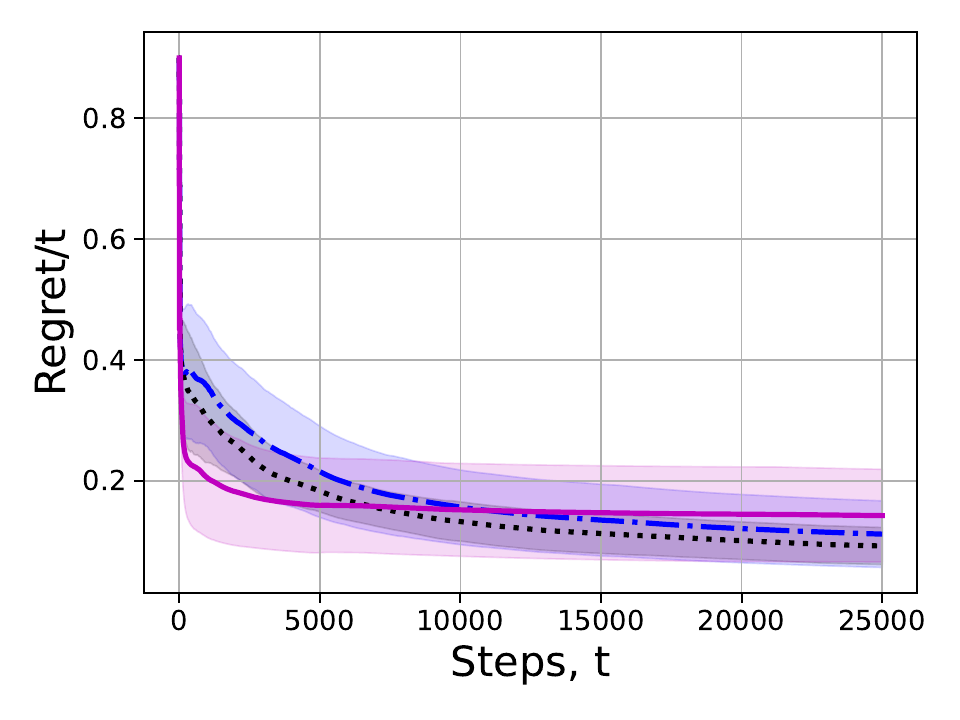}
\caption{mean regret}
\end{subfigure}
~
\caption{Dependence of cumulative regret (left) and mean regret (right) for the considered algorithms in the multi-armed bandit (MAB) setting with Cauchy noise. The values are averaged over 30 trials, with the shaded regions representing the corresponding standard deviation. The graphs indicate that the proposed functional version shows comparable regret to the original algorithm.}
\label{fig::mab_experiment}

\end{figure}

\section{Lower bounds for deterministic setups}\label{sec: Lower bounds}

Next we show that lower bounds for BFI problem setting could be derived from lower bounds for related optimization problems. 
We establish them for deterministic algorithms.
Next we proceed with a general reduction scheme to obtain lower bounds for our problem based on lower bounds for respective optimization problems. 

Denote by $\mathcal{P}_i$ ($1\leq i \leq k$) optimization problem 
$$
\min_{x\in \mathcal{X}_i} f_i(x),
$$
and by $f_i^*$ the value of the objective at the solution. For simplicity reasons, we assume that objective functions are from the same class $f_1, \dots, f_k \in \mathcal{F}$ and that each function is equipped with the oracle from the same class and feasibility sets $\mathcal{X}_i$ are nonempty convex sets with the same difficulty for oracles. The complexity of each $\mathcal{P}_i$ (in terms of~\cite{nemirovskij1983problem}) is the same, i.e. each problem has the same lower bound. To make it more precise, there exists a function $g(H, t)$ such that for each algorithm $\mathcal{A}$ and resulting  sequence of feasible test points $\{x_t\}_{t=1}^{\infty}$ defined as $x_{t+1} = \mathcal{A}\left (x_1, \mathcal{O}(x_1), \dots, x_t, \mathcal{O}(x_t) \right)$ there exists a problem instance $\mathcal{P}$ defined by parameters $H = (f, \mathcal{X})$ and oracle~$\mathcal{O}$ the following inequation holds:
\begin{equation}\label{lower_bound}
    f(x_t) - \min_{x\in \mathcal{X}} f(x) \geq g(H, t).
\end{equation}

Function $g(H, t)$ is much easier to comprehend if it can be factorized as $g(H, t) = \phi(H) \cdot t^{-\alpha}$ and encapsulated as $\underline g(t) = \inf_{H} g(H, t)$ and then used $\Omega (t^{-\alpha})$ notation instead. For example, in the case of unconstrained minimizing a smooth convex objective equipped with a first-order oracle, in the seminal work \cite{nemirovskij1983problem} shows that optimal rates are $t^{-2}$. We will call lower bound \textit{sharp and uniform} for an algorithm family $\mathcal{W} = \{\mathcal{A}\}$ and problem family $\mathcal{S} = \{\mathcal{P}\}$ such that for any algorithm $\mathcal{A}$ from $\mathcal{W}$ there exists a problem $\mathcal{P}$ such that the following equality holds $f(x_t) - \min_{x\in \mathcal{X}} f(x) \geq g(H_{\mathcal{P}}, t)$, where $x_{t+1} = \mathcal{A}(x_1, \mathcal{O}(x_1), \dots, x_t, \mathcal{O}(x_t))$, for each time $t>0$.

We also need to introduce the concept of optimal algorithm, i.e. the algorithm $\mathcal{A}^*$ with convergence rate same as lower bound up to constant: 
\begin{equation}
    f(x_t) - \min_{x\in \mathcal{X}} f(x) \leq C_{\mathcal{A}^*} g(H, t),
\end{equation}
where $C_{\mathcal{A}^*}>0$ is some known constant.

For the problem of unconstrained minimization of smooth convex objectives equipped with a first-order oracle, the renowned Nesterov's accelerated gradient descent (AGD) \cite{nesterov1983method} is optimal. Actually, monotone version of AGD is needed, since we introduce \textit{uniform} lower bound, that holds for each time step $t$.


Assuming that the agent has access to the optimal algorithm $\mathcal{A}$ for the problem class of problems $\mathcal{P}_1, \dots, \mathcal{P}_k$ with exact lower bound defined by the function $g(H, t)$, what is the lower bound for the optimal arm identification problem?

More formally, what number of testing points $T$ is needed to certify that $R(T) \leq \varepsilon$?

Let $\mathcal{P}_1$ and $\mathcal{P}_2$  be two problem instances from the same class with optimal values $f_1^*$ and $f_2^*$ respectively. 

Denote by $t_i(\varepsilon) = t(\varepsilon, H_i) = \{\min t \mid g(H_i, t) \leq \varepsilon \}$ the \textit{vicinity hitting time} for the problem defined by $\mathcal{H}_i$ and tolerance level $\varepsilon$. In case of sharp lower bounds we assume that for each algorithm $\mathcal{A}$ there exists \textit{hard} problem $\mathcal{P}$ such that if $t > t_i(\varepsilon)$, then $f(x_t) - f^* < \varepsilon$.

Let $\mathcal{A}$ be an algorithm supported by the oracle $\mathcal{O}(x)$ and the sequence of points it generates $\{x_t\}_{t=1}^{\infty}$ such that 
\begin{equation*}
x_{t+1} = \mathcal{A}\left (x_1, \mathcal{O}(x_1), \dots, x_t, \mathcal{O}(x_t) \right) \quad  \text{for all}~ t.
\end{equation*}

Denote by $\mathcal{L} (\mathcal{G}) = \{\mathcal{L}_{t}(\mathcal{G})\}_{t=1}^{\infty}$ the sequence of problem sets, such that 
\begin{equation*}
\mathcal{L}_{t}(\mathcal{G}) = \{\mathcal{P} \quad | \quad x_s^{\mathcal{P}} = x_s^{\mathcal{G}}, \quad  \mathcal{O}_{\mathcal{G}}(x_s^{\mathcal{G}}) = \mathcal{O}_{\mathcal{P}}(x_s^{\mathcal{P}}), \quad 1\leq s \}
\end{equation*}

, i.e. $\mathcal{L}_{t}$ is a set of problems that could not be distinguished by any algorithm $\mathcal{A}$ with supporting oracle~$\mathcal{O}$ until step $t$.

Let $\mathcal{P}_1$ and $\mathcal{P}_2$  be two problem instances from the same class. 

\begin{lemma}
Consider the best function identification problem consisting of two  problems from a given class with uniform sharp lower bound function $g(H, t)$. Then achieving regret $R_B(T) \leq \varepsilon$ requires at least
$$
T \geq  2 \cdot t(\varepsilon)
$$
testing points.
\end{lemma}

\begin{proof}
Note that to certify that $R_B(T) \leq \varepsilon$ one needs either to show that $|f_1^* - f_2^*| \leq \varepsilon$ or to identify the index of the optimal function $i^* = \arg \min_{i={1, 2}} f_i^*$.

Suppose that the algorithm tries to solve the best arm identification problem for $\mathcal{P}_1$ and $\mathcal{P}_2$, runs $\hat{t}_1$ and $\hat{t}_2$ iterations for problem $\mathcal{P}_1$ and $\mathcal{P}_2$ respectively and then makes a decision. Then it must make the same decision for any problem pair $(\hat{\mathcal{P}}_1, \hat{\mathcal{P}}_2)$ such that $\hat{\mathcal{P}}_1 \in L_{\hat{t}_1} (\mathcal{P}_1)$ and $\hat{\mathcal{P}}_2 \in L_{\hat{t}_2} (\mathcal{P}_2)$.

Let $\mathcal{P}$ be \textit{hard} problem for a given algorithm $\mathcal{A}$, i.e. such that inequality \ref{lower_bound} holds. Let $f^*$ be the optimal value in $\mathcal{P}$ and $t(\varepsilon)$ denote vicinity hitting time for $(\mathcal{A}, \mathcal{P})$.

Suppose that $\mathcal{P}_1 \in \mathcal{L}_{\hat{t}_1-1} (\mathcal{P})$ and $\mathcal{P}_2 \in \mathcal{L}_{\hat{t}_2-1} (\mathcal{P})$, i.e. problems could not be distinguished by the algorithm.

Next let us fix \textit{computational budgets} and show that there exist setups such that it would not be sufficient to guarantee $R(T) \leq \varepsilon$.

Let us fix any positive numbers $\hat{t}_1$ and $\hat{t}_2$ such that $\hat{t}_1+\hat{t}_2 = T$ and assume that $2\cdot t(\varepsilon) > T$. Then at least one inequality holds: $\hat{t}_1<t(\varepsilon)$, or $\hat{t}_2<t(\varepsilon)$, or both $\hat{t}_1<t(\varepsilon)$ and $\hat{t}_2<t(\varepsilon)$.

W.l.o.g. assume that $\hat{t}_1<t(\varepsilon)$. 

Suppose that the algorithm picks $f_1$ as a solution. Then we can pick feasible problem pair $\mathcal{P}_1 \in \mathcal{L}_{\hat{t}_1 - 1}(\mathcal{P})$, $f_1^* = f_2(x_{\hat{t}_1^{\mathcal{P}}})$ and $\mathcal{P}_2 = \mathcal{P}$. For this problem pair it is known that the optimal arm is $\mathcal{P}_2$, and regret could be estimated as follows.
\begin{equation*}
R_B(T) = f_1^* - f_2^* = f_2(x_{\hat{t}_1^{\mathcal{P}}}) - f^*_2 \geq g(\mathcal{P}, \hat{t}_1) >  \varepsilon
\end{equation*}

Suppose that the algorithm picks $f_2$ as a solution. Then we could pick a feasible problem pair $\mathcal{P}_1 \in \mathcal{L}_{\hat{t}_1 - 1}(\mathcal{P})$, $f_1^* < f^* -  \varepsilon$ and $\mathcal{P}_2 = \mathcal{P}$. In this case problem $\mathcal{P}_1$ masks itself as $\mathcal{P}$ until round $\hat{t}_1$ and is not restricted after.
\end{proof}

\begin{conjecture}
    The best function identification problem with $K$ arms requires at least $T \geq K \cdot t(\varepsilon)$ to guarantee $R_B(T) \leq \varepsilon$.
\end{conjecture}

For standard problem setups $t(\varepsilon)$ could be estimated explicitly.

\paragraph{Explicit forms of $t(\varepsilon)$.}

\begin{itemize}
    \item Convex $M$-Lipschitz:
  $\underline g(t)=\Omega(MR/\sqrt{t})
  \ \Rightarrow\ t(\varepsilon)=\Omega\!\bigl((MR/\varepsilon)^2\bigr)$,
  \item Convex $M$-Lipschitz: $\underline g(t)=\Omega(MR/\sqrt{t})
  \ \Rightarrow\ t(\varepsilon)=\Omega\!\bigl((MR/\varepsilon)^2\bigr)$,
  \item $\mu$-strongly convex, $M$-Lipschitz: $\underline g(t)=\Omega(M^2/(\mu t))
  \ \Rightarrow\ t(\varepsilon)=\Omega\!\bigl(M^2/(\mu\varepsilon)\bigr)$,
  \item $\mu$-strongly convex, $L$-smooth: $\underline g(t)=\Omega(R^2 e^{-t/\sqrt{\kappa}})
  \ \Rightarrow\ t(\varepsilon)=\Omega\!\bigl(\sqrt{\kappa}\,\log(R^2/\varepsilon)\bigr)$.
\end{itemize}

\begin{theorem}[Minimax lower bound for nonstochastic BFI]
\label{thm:lb-bfi}
For any BFI algorithm and any $T\in\mathbb{N}$ there exists a family
$\{P_i\}_{i=1}^K \subset \mathcal{F}$ such that
\[
  R_B(T)\ \ge\ \underline g^{-1}\left ( \frac{T}{K} \right ),
\]
where $\underline g^{-1}(t)$ satisfies $\underline g( \underline g^{-1}(t)) = t$.
\end{theorem}

\begin{proof}
    Consider BFI instance $\mathcal{Z}_m = (\mathcal{P}_1, \dots, \mathcal{P_K})$. We call this instance 'hard' if \begin{itemize}
        \item $\mathcal{P}_i = \mathcal{P}_j, \quad \forall i, j\neq m$
        \item for each $i$ ($i\in 1\dots, K$) the corresponding problem $\mathcal{P}_i$ is ''hard'' for base optimizer, i.e. it satisfies $f_i(x^{t, i}) - f_i^* \geq \underline g(t)$, $x^{t+1, i} = \mathcal{A}_i(x^{1, i}, \mathcal{O}_i (x^{1, i}), \dots, x^{t, i}, \mathcal{O}_i (x^{t, i}))$ for every $t\in 1, \dots, T$.
        \item For every $t \leq \frac{T}{K}$ and any $i, j \in 1\dots, K$ holds 
        \begin{equation*}
        \mathcal{A}_i(x^{1, i}, \mathcal{O}_i (x^{1, i}), \dots, x^{t, i}, \mathcal{O}_i (x^{t, i})) = \mathcal{A}_j(x^{1, j}, \mathcal{O}_j (x^{1, j}), \dots, x^{t, j}, \mathcal{O}_j (x^{t, j})).  
        \end{equation*}
        \item Corresponding optimal values satisfy $f_i^* -f_m^* = \underline g^{-1} \left ( \frac{T}{K}\right ), \quad \forall i\neq m$. 
    \end{itemize} 
In simple terms, the instance is hard if each individual problem is hard for the corresponding base optimizer, and objective values are hard to distinguish by any algorithm. 

Then for any deterministic procedure, there exists an arm $n$ with number of pools $k_n(T) \leq \frac{T}{K}$ and this arm could not be distinguished from other arms. 
If the procedure chooses $J_T = n$, then for any $Z_m (m\neq n)$ instance the regret is  $R_B(T)\ \ge\ \underline g^{-1}\left ( \frac{T}{K} \right )$. 
If the procedure chooses $J_t \neq n$, then  $R_B(T)\ \ge\ \underline g^{-1}\left ( \frac{T}{K} \right )$ for instance $Z_n$ by construction.

\end{proof}

\subsection{Lower Bound for FMAB}

Consider the FMAB problem with cumulative regret
$R_O(T)=\sum_{t=1}^T \bigl(f_{i_t}(x_{t,i_t})-f^*)$.
Let $k_i(T)$ denote the number of oracle calls to subproblem $f_i$ up to time $T$,
and set $G(m):=\sum_{s=1}^m \underline g(s)$.

\begin{theorem}[Minimax lower bound for FMAB]
\label{thm:lb-fmab_app}
For any FMAB algorithm and any $T\in\mathbb{N}$ there exists a family
$\{P_i\}_{i=1}^K \subset \mathcal{F}$ such that
\[
  R_O(T)\ \ge\ \inf_{\{k_i\ge 0:\ \sum_{i=1}^K k_i=T\}}
  \ \sum_{i=1}^K G\bigl(k_i\bigr).
\]
In the homogeneous case (all $P_i$ have the same hardness $\underline g$),
this implies the following orders:
\begin{itemize}
    \item (i) Convex $M$-Lipschitz: $\underline g(s)=\Omega(MR/\sqrt{s}) \Rightarrow G(m)=\Omega(MR\sqrt{m})$ 
    
    $\Rightarrow R_O(T)=\Omega \left (MR\sqrt{T}\right )$,
    \item (ii) $L$-smooth convex: $\underline g(s)=\Omega(LR^2/s^2)
    \Rightarrow G(m)=\Omega(LR^2) \Rightarrow R_O(T)=\Omega( LR^2 )$,
  \item (iii) $\mu$-strongly convex, $M$-Lipschitz: $\underline g(s)=\Omega\!\bigl(M^2/(\mu s)\bigr)
   \ \Rightarrow\ G(m)=\Omega\!\bigl((M^2/\mu)\log m\bigr) \Rightarrow R_O(T)=\Omega\!\bigl((M^2/\mu)\log T\bigr)$,
   \item (iv) $\mu$-strongly convex, $L$-smooth: $\underline g(s)=\Omega\!\bigl(R^2 e^{-s/\sqrt{\kappa}}\bigr)
    \Rightarrow G(m)=\Omega(R^2) \Rightarrow\ R_O(T)=\Omega( R^2 )$.
\end{itemize}
\end{theorem}

\begin{proof}[Proof sketch]
We will use the same scheme with ''hard'' instances as in the previous section. Construct ``hard'' instances $P_i$ so that all $f_i$ share the same minimizer
$f^*$ and have identical oracle complexity $\underline g(\cdot)$.
Then, each time an arm $i_t$ is queried, the incurred regret is at least
$\underline g(k_{i_t})$. Summing over time gives
\begin{equation*}
        R_O(T) = \sum_{t=1}^T \left [f_{i_t}(x^{t, i_t}) - f^* \right ] = \sum_{i=1}^K \sum_{s=1}^{k_i} \left  [f_{i}(x^{s, i}) - f^* \right ] \geq \sum_{i=1}^K\sum_{s=1}^{k_i}\underline g(s)=\sum_{i=1}^K G(k_i).
\end{equation*}
Since the algorithm itself chooses $\{k_i\}$ under the constraint $\sum_i k_i=T$,
this yields the infimum formulation.
For cases (i) and (iii), $G$ is concave, so the infimum is attained by
concentrating queries on a single subproblem, leading to the claimed orders. 
\end{proof}

\begin{remark}
    Provided lower bound is tight in terms of $T$. But what about $K$? In this case, one could consider $K$ ''hard'' and indistinguishable for algorithm instances. Then for any class of algorithms, that \textit{explore fairly} (i.e. if two arms are indistinguishable, then if an algorithm chooses to query one arm, it must also query another arm) holds:
    $$
    R_O(T) \sim K G\left ( \frac{T}{K}\right ),
    $$ which results in $\Omega\left ( MR \sqrt{KT}\right )$ for convex $M$-Lipschitz functions.
\end{remark}

\end{document}